\documentclass{article}

\PassOptionsToPackage{numbers,compress}{natbib}
\usepackage[final,main]{neurips_2026}

\usepackage[utf8]{inputenc}
\usepackage[T1]{fontenc}
\usepackage{hyperref}
\usepackage{xurl}
\usepackage{url}
\usepackage{booktabs}
\usepackage{array}
\usepackage{ragged2e}
\usepackage{amsfonts}
\usepackage{amsmath}
\usepackage{amssymb}
\usepackage{amsthm}
\usepackage{nicefrac}
\usepackage{microtype}
\usepackage{xcolor}
\usepackage{graphicx}
\usepackage{subcaption}
\usepackage{multirow}
\usepackage{enumitem}
\usepackage[ruled,vlined,linesnumbered]{algorithm2e}
\usepackage{makecell}
\usepackage{booktabs}
\usepackage{colortbl}

\theoremstyle{definition}
\newtheorem{definition}{Definition}[section]
\newtheorem{assumption}[definition]{Assumption}

\theoremstyle{plain}

\newtheorem{proposition}[definition]{Proposition}
\newtheorem{lemma}[definition]{Lemma}

\theoremstyle{remark}
\newtheorem{remark}[definition]{Remark}

\graphicspath{{../build/paper/figures/}{figures/}}

\title{$\boldsymbol{f}$-OPD: Stabilizing Long-Horizon On-Policy Distillation with Freshness-Aware Control}

\author{Xianwei Chen, Shimin Zhang, \textbf{Jibin Wu}\thanks{Corresponding author, jibin.wu@polyu.edu.hk}\\ 
~\\
The Hong Kong Polytechnic University
}

\begin{document}

\maketitle

\begin{abstract}

Scaling on-policy distillation (OPD) for large language models (LLMs) confronts a fundamental tension: asynchronous execution is necessary for system efficiency, yet it structurally deviates from the ideal on-policy objective.
To address this challenge, we theoretically decompose the objective discrepancy into \emph{rollout drift} and \emph{supervision drift}, capturing staleness in student rollout and teacher context, respectively.
Building on this, we introduce a sample-level freshness score that quantifies the reliability of a buffered sample with respect to the on-policy objective. Guided by this signal, we further propose \textbf{$\boldsymbol{f}$-OPD}, a novel framework that adaptively regulates the influence of stale samples and constrains policy drift accumulated under asynchronous training. 
Across reasoning, tool-use, and coding-agent tasks of increasing interaction horizon, $f$-OPD consistently achieves task performance comparable to synchronous optimization while largely retaining the throughput advantages of asynchronous execution.
%
Our results establish the first recipe for achieving a performance–efficiency trade-off in OPD, paving the way for long-horizon agentic post-training at scale.















\end{abstract}

\section{Introduction}
The recent leap in large language models (LLMs) has been driven not only by pretraining but increasingly by advances in post-training. Over the past few years, paradigms such as supervised fine-tuning (SFT) and reinforcement learning (RL) have fundamentally reshaped how pretrained models are aligned, specialized, and deployed~\cite{touvron2023llama2,schulman2017proximal,shao2024deepseekmath}. Most recently, on-policy distillation (OPD), by unifying on-policy learning with dense supervision, has emerged as a powerful technique, driving substantial gains in frontier LLMs~\cite{yang2025qwen3,deepseek2026deepseekv4,xiao2026mimov2flash,zeng2026glm5}.

A standard OPD pipeline consists of three tightly coupled processes: student policy rollout, teacher supervision for grading generated tokens, and student policy optimization. As shown in Fig.~\ref{fig:Fig1}(a), these processes should be executed synchronously in principle so that each training step remains fully on-policy. Empirically, however, strict synchronization leads to substantial system under-utilization, especially in long-horizon settings where rollout and supervision become increasingly expensive. As a result, modern OPD systems~\cite{sheng2024hybridflow,zhu2025slime,nemo-rl} widely adopt asynchronous mode (Fig.~\ref{fig:Fig1}(b)) to improve throughput by overlapping different stages of the pipeline and executing them in parallel. However, this efficiency gain comes at the cost of \emph{staleness}: buffered trajectories and their associated supervision may no longer match the current student policy when used for optimization. Such a mismatch in freshness can distort the effective on-policy objective, destabilize optimization, and ultimately degrade task performance. This creates a fundamental \emph{performance--efficiency trade-off} in OPD: improving system utilization through asynchrony often deteriorates optimization fidelity, while preserving on-policy consistency requires sacrificing throughput.

\begin{figure}[ht]
    \centering
    \includegraphics[width=\linewidth]{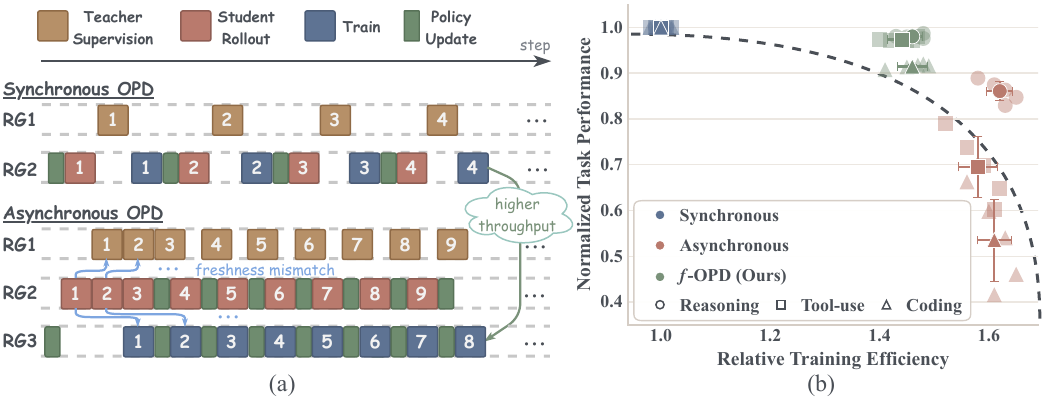}
    \caption{(a) System implementation of OPD under synchronous (top) and asynchronous (bottom) execution. While asynchronous OPD improves training throughput by pipelining computation across multiple resource groups (RGs), the resulting freshness mismatch compromises training fidelity. (b) Performance–efficiency trade-offs across reasoning, tool-use, and coding-agent tasks. Points denote five-seed means with standard deviations. While vanilla OPD executions struggle to simultaneously balance performance and efficiency, \textit{f}-OPD effectively bridges this gap by freshness-aware control.}
    \label{fig:Fig1}
\end{figure}

\textcolor{black}{Despite the fact that freshness mismatch has been well studied in asynchronous RL~\cite{mnih2016asynchronous,espeholt2018impala,liu2025when,hilton2021batch}, its implications for OPD remain underexplored due to asynchronous OPD execution extending beyond the conventional decoupling of data collection and optimization.
An intuitive way is to periodically refresh the replay buffer\footnote{A buffer of previously collected rollouts and their associated teacher supervision used for subsequent policy updates.} maintained during training, preventing stale data from continuously corrupting policy updates.}
Yet such a hard refresh introduces its own dilemma: refreshing too aggressively discards potentially useful learning signals, while refreshing too conservatively allows stale optimization to accumulate. This tension suggests that staleness should not be treated as a buffer-level binary (pick or drop) trigger for refreshing, but instead be quantified at the level of individual training samples. A natural online proxy for the staleness is policy update lag, which denotes the number of optimization steps elapsed since a sample was generated. However, lag alone is insufficient in the OPD: samples with identical ages can exhibit vastly different mismatches due to heterogeneous drift from both the student rollout policy and teacher supervision.

In this paper, we formulate the freshness mismatch in OPD as an objective discrepancy between asynchronous execution and its ideal on-policy counterpart. Through theoretical analysis, we show that this discrepancy is driven by two structurally distinct sources: \emph{rollout drift}, arising from stale student policy occupancy, and \emph{supervision drift}, arising from outdated teacher conditioning. Based on this, we derive a fine-grained characterization of sample-level staleness and introduce a \textbf{freshness score} that quantifies each sample’s fidelity to the on-policy objective. We further propose \textbf{$\boldsymbol{f}$-OPD}, a freshness-aware framework that compensates for objective discrepancy while actively correcting the policy drift induced by staleness. Our main contributions are summarized as follows:
\begin{itemize}[leftmargin=1.5em]
    \item We formalize asynchronous OPD as optimization under freshness mismatch, and theoretically characterize the resulting objective discrepancy as arising jointly from both the student and teacher models. Based on this decomposition, we introduce a fine-grained metric for tracking sample-level fidelity to the on-policy distillation objective.

    \item We propose ${f}$-OPD, a freshness-aware OPD framework that bridges the gap between efficient asynchronous execution and ideal on-policy optimization.

    \item We conduct comprehensive experiments on a range of tasks, including reasoning, tool-use, and coding-agent tasks with progressively longer interaction horizons, and demonstrate that ${f}$-OPD consistently achieves a favorable balance between task quality and system throughput.
\end{itemize}

\section{Related Work}
\textbf{LLM Post-training and On-Policy Distillation.}
Post-training has become the primary mechanism for eliciting task competence and interactive reasoning in LLMs. Existing paradigms exhibit complementary trade-offs. SFT~\cite{bengio2015scheduled} optimizes forward-KL objectives over fixed datasets, offering highly stable token-level supervision but limited exploration beyond the training distribution. RL~\cite{schulman2017proximal,shao2024deepseekmath}, in contrast, learns from self-generated trajectories with reward-driven objectives, enabling stronger exploration~\cite{yu2025dapo,cui2025entropy,zhang2025relax} at the cost of unstable optimization and potential policy collapse. Positioned between these two paradigms, OPD~\cite{song2026survey,gu2024minillm,agarwal2024onpolicy,li2026rethinking} combines on-policy learning with distribution-level teacher supervision, inheriting RL's exploratory capability while preserving the dense, stable learning signal of distillation. Recent on-policy self-distillation~\cite{hubotter2026reinforcement,zhao2026selfdistilledreasoner,zhang2026opsdl,sang2026crisp} and online teacher-supervised variants~\cite{ko2026scaling,ye2026onpolicy,yang2026learning,fu2026revisiting,ye2025blackbox,jin2026entropy,zhang2026fast} further demonstrate the promise of this paradigm for LLM post-training. 
Despite this broader progress, a central challenge remains unresolved: maintaining the statistical fidelity of on-policy optimization while achieving the system efficiency required for large-scale OPD training.

\textbf{Asynchronous Systems and Freshness Mismatch.}
Asynchronous learning systems are widely adopted to improve utilization by decoupling data collection and optimization. Asynchronous reinforcement learning systems~\cite{mnih2016asynchronous,espeholt2018impala} have become a standard paradigm for scaling training by decoupling experience collection from policy optimization, enabling higher throughput, better resource utilization, and large-scale distributed execution. However, this decoupling also leads to stale trajectories when acting and learning evolve under inconsistent policy states. Similar staleness effects appear in distributed optimization and parallel training~\cite{zhang2016staleness,harlap2019pipedream}. In policy-gradient methods, this mismatch can already be viewed through the classical on-policy/off-policy tension underlying REINFORCE~\cite{williams1992simple}, and it becomes especially visible in modern LLM-RL, where recent analyses link training collapse directly to training-inference mismatch~\cite{liu2025when}. Prior work further shows that such a mismatch can compound across sequential rollouts, motivating drift-control principles such as dataset aggregation and trust-region optimization~\cite{ross2011reduction,schulman2015trust,schulman2017proximal}. More recent long-horizon RL work sharpens this connection by studying explicit trust-region masking~\cite{li2025trust} and by showing that decoupled PPO can separate the proximal policy used to control update size from the behavior policy used for off-policy correction, thereby making better use of stale data~\cite{hilton2021batch}. Our setting extends this line of work to teacher-supervised OPD, where freshness mismatch arises jointly from stale student rollouts and outdated teacher supervision.

\section{Theoretical Analysis}
\label{sec:theory}
Although the freshness mismatch arises mainly from systems-level asynchrony, it can be analyzed as an objective-level discrepancy. In this section, we formalize the gap between the ideal on-policy distillation objective and the objective optimized under an asynchronous pipeline, and show how this discrepancy arises from both the student and the teacher.

\subsection{Objective Discrepancy under Different OPD Training Modes}

We begin by distinguishing training objectives under synchronous and asynchronous training pipelines. 
Let $\pi_{\theta}^t$ denote the student policy at update step $t$, and let $\pi_{\mathrm{teacher}}(\cdot \mid x)$ denote the teacher distribution conditioned on prefix $x$. The ideal synchronous OPD objective is defined as
\begin{equation}
\mathcal{J}_{\mathrm{sync}}(\pi_{\theta}^t) =
\mathbb{E}_{x \sim d^{t}}
\left[
\ell\bigl(\pi_{\theta}^t(\cdot \mid x), \pi_{\mathrm{teacher}}(\cdot \mid x)\bigr)
\right],
\end{equation}
where $d^{t}$ denotes the mixture distribution over trajectory prefixes induced by rolling out $\pi_{\theta}^t$ across token positions, and $d_h^{t}$ denotes the corresponding depth-$h$ prefix distribution when depth-wise arguments are needed below.
Under asynchronous execution, training proceeds on an active buffer $\mathcal{B}^t$ composed of samples generated and graded at current or earlier policy update steps. For each sample $i \in \mathcal{B}^t$, let $r(i) \le t$ denote the step at which the sample was produced, let $x_i$ denote the stored prefix, and let $c_i^{r(i)}$ denote the teacher-conditioning context actually used when the sample was graded. The corresponding objective implemented by the asynchronous pipeline can be formulated as
\begin{equation}
\mathcal{J}_{\mathrm{async}}(\pi_{\theta}^t) =
\mathbb{E}_{i \sim \mathcal{B}^t}
\left[
\ell\bigl(\pi_{\theta}^t(\cdot \mid x_i), \pi_{\mathrm{teacher}}(\cdot \mid c_i^{r(i)})\bigr)
\right].
\end{equation}
Since student rollout and teacher supervision may be conditioned on different information in the asynchronous pipeline, which is context-dependent, we distinguish them throughout this paper by notation. Specifically, $x$ denotes the student-side prediction prefix, while $c$ denotes the teacher-side conditioning context used to generate supervision. In the simplest prefix-only setting, $c=x$, whereas in general $c\neq x$. Accordingly, $x_i$ denotes the stored sample prefix, $c_i^{r(i)}$ the teacher context used at labeling time, and $c_i^{t}$ the teacher context that would be used if the same sample were relabeled at update step $t$. We then define the asynchronous objective discrepancy at $t$ as
\begin{equation}
\Delta^t = \bigl|\mathcal{J}_{\mathrm{async}}(\pi_{\theta}^t) - \mathcal{J}_{\mathrm{sync}}(\pi_{\theta}^t)\bigr|.
\end{equation}

\subsection{Two-Fold Decomposition of the Asynchronous Objective Discrepancy}
Unlike reward-estimation bias in off-policy RL, $\Delta^t$ arises from optimizing over prefixes generated by older student policies and potentially, on teacher supervision tied to outdated conditioning contexts. We next theoretically decompose these two sources of discrepancy under the following assumption.

\begin{assumption}[Truncation and regularity]
\label{assumption:regularity}
We evaluate the distillation loss on a finite support where both student and teacher probabilities are bounded away from zero by a positive constant. On this domain, $\ell(\cdot,\cdot)$ is nonnegative and Lipschitz with respect to total-variation distance.
\end{assumption}

To isolate the student- and teacher-sides of freshness mismatch, we first introduce an intermediate distillation objective that pairs stale student rollouts with up-to-date teacher conditioning contexts:

\begin{equation}
\label{eq:intermediate}
\widetilde{\mathcal{J}}_{\mathrm{async}}(\pi_{\theta}^t)
=
\mathbb{E}_{x\sim d^{r(i)}}\bigl[\ell\bigl(\pi_{\theta}^t(\cdot\mid x), \pi_{\mathrm{teacher}}(\cdot\mid x)\bigr)\bigr],
\end{equation}
where $d^{r(i)}$ denote the prefix distribution of sample $i$ induced by the stale policy.
We then define the buffer-level stale distributions by $d^{\mathrm{stale},t} := \mathbb{E}_{i\sim\mathcal{B}^t}[d^{r(i)}]$ and yield the following proposition for discrepancy decomposition.

\begin{proposition}[Two-fold Decomposition for Asynchronous Objective Discrepancy]
\label{prop:async-discrepancy}
Under Assumption~\ref{assumption:regularity}, there exist nonnegative constants $C_{\mathrm{roll}}$ and $C_{\mathrm{sup}}$ such that:
\begin{equation}
\label{eq:decomposition}
\Delta^t
\le
\underbrace{
C_{\mathrm{roll}}\,
\mathrm{TV}\!\left(d^{t},\,d^{\mathrm{stale},t}\right)
}_{\text{rollout drift}}
+
\underbrace{
C_{\mathrm{sup}}\,
\mathbb{E}_{i\sim\mathcal{B}^t}
\!\left[
\mathrm{TV}\!\left(\pi_{\mathrm{teacher}}(\cdot \mid c_i^{t}), \, \pi_{\mathrm{teacher}}(\cdot \mid c_i^{r(i)})\right)
\right]
}_{\text{supervision drift}},
\end{equation}
where $\mathrm{TV}(\cdot,\cdot)$ denotes the total-variation distance. The detailed proof of decomposition is provided in the Appendix~\ref{appendix-sec:proof-decomposition}.
\end{proposition}
By the convexity of $\mathrm{TV}(\cdot,\cdot)$:
\begin{equation}
    \mathrm{TV}\!\left(d^{t},\,d^{\mathrm{stale},t}\right)
\le
\mathbb{E}_{i\sim\mathcal{B}^t}\!\left[\mathrm{TV}\!\left(d^{t},\,d^{r(i)}\right)\right],
\end{equation}
which recasts the rollout drift term in Eq.~\ref{eq:decomposition} into a buffer-level expectation to align with the supervision drift term.
Thus, the objective discrepancy in asynchronous OPD is decomposed into two structurally distinct terms: \emph{\textbf{rollout drift}}, arising when the active buffer reflects prefixes drawn from the older student policy rather than the current policy; and \emph{\textbf{supervision drift}}, arising when teacher supervision is attached to outdated conditioning contexts that no longer align with the current contexts.

\subsection{Diagnostics for Tracking Staleness}
\label{sec:diagnostics}
While the above decomposition identifies the sources of staleness and formally characterizes their contributions, the resulting drift terms are not directly observable during optimization. We therefore introduce sample-level diagnostics that empirically track them through the distributions of the student and teacher models. 
Specifically, for rollout drift, we use an on-support divergence evaluated on buffered prefixes:
\begin{equation}
\delta_i^{\mathrm{roll}} = \mathrm{TV}\!\left(\pi_{\theta}^{t}(\cdot\mid x_i),\, \pi_{\theta}^{r(i)}(\cdot\mid x_i)\right),
\qquad
D_i^{\mathrm{roll}} = \mathrm{KL}\!\left(\pi_{\theta}^{t}(\cdot\mid x_i)\,\|\,\pi_{\theta}^{r(i)}(\cdot\mid x_i)\right).
\end{equation}
By Pinsker's inequality, $\delta_i^{\mathrm{roll}} \le \sqrt{\tfrac{1}{2}\, D_i^{\mathrm{roll}}}$, so the empirical KL serves as a conservative upper bound on the corresponding total-variation distance.
In sequential decision-making settings, however, even local policy drift can compound over time, amplifying stale occupancy. We provide analysis for this effect in Proposition~\ref{prop:horizon-compounding}.
As for supervision drift, we can also derive an upper bound by Pinsker's inequality as (see Appendix~\ref{appendix-sec:sup_drift_analysis} for details):
\begin{equation}
    \delta_i^{\mathrm{sup}} = \mathrm{TV}\!\left(\pi_{\mathrm{teacher}}(\cdot\mid c_i^t),\, \pi_{\mathrm{teacher}}(\cdot\mid c_i^{r(i)})\right)
    \le \sqrt{\tfrac{1}{2}\,D_i^{\mathrm{sup}}},
\end{equation}
where $D_i^{\mathrm{sup}} = \mathrm{KL}\!\left(\pi_{\mathrm{teacher}}(\cdot \mid c_i^{t})\,\|\,\pi_{\mathrm{teacher}}(\cdot \mid c_i^{r(i)})\right)$ denotes the divergence between the teacher distribution under the current and outdated contexts.
In summary, $D_i^{\mathrm{roll}}$ and $D_i^{\mathrm{sup}}$ serve as monotone surrogates rather than exact estimators of the discrepancy terms, but provide observable signals that can be incorporated into freshness-aware control.
Together with policy update lag, we illustrate in Figure~\ref{fig:theory-diagnostics-schematic} (top) three complementary diagnostics for characterizing sample-level staleness. These diagnostics form the foundation for incorporating freshness-aware control into OPD optimization.

\begin{figure}[ht]
\centering
\includegraphics[width=\linewidth]{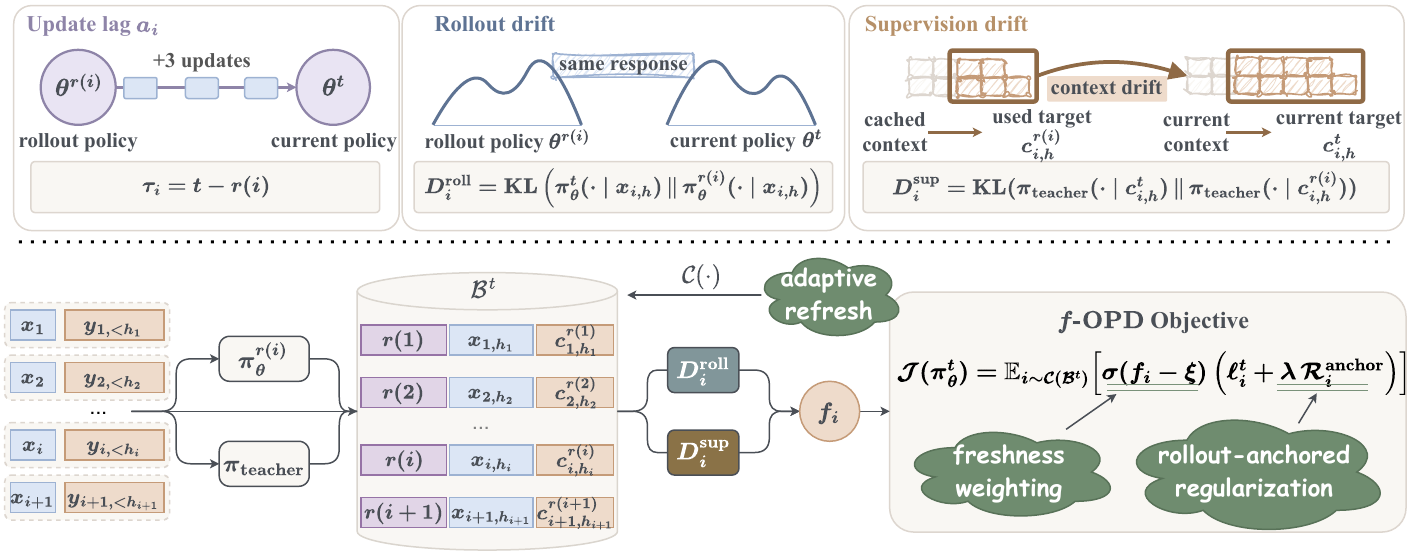}
\caption{Systematic overview of \textit{f}-OPD. \emph{Top}: three sample-level diagnostics used to characterize staleness. \emph{Bottom}: the overall \textit{f}-OPD pipeline, where sample freshness is estimated from these diagnostics and integrated into OPD optimization through three complementary mechanisms.}
\label{fig:theory-diagnostics-schematic}
\end{figure}

\section{Freshness-Aware Control for Long-Horizon OPD}
\label{sec:method}
Our analysis in Sec.~\ref{sec:theory} has decomposed asynchronous OPD staleness into rollout and supervision drift and introduced the corresponding diagnostics $D_i^{\mathrm{roll}}$ and $D_i^{\mathrm{sup}}$. 
In this section, we define a sample-level freshness score based on these diagnostics, and further propose ${f}$-OPD with three complementary mechanisms for freshness-aware control.

\subsection{Fine-Grained Sample Freshness Scoring}
To evaluate how stale each buffered sample is, each buffer must retain sufficient provenance to replay the trajectory under both the current and the stale policies and to reconstruct the teacher-conditioning contexts used now and at grading time. For each sample $i$, we maintain its actual rollout step $r(i)$, the stored prefix $x_i$, and the token-level information $\{(y_{i,h},\, c_{i,h}^{r(i)})\}_{h=1}^{H_i}$, where $H_i$ denotes the token length of sample $i$, $y_{i,h}$ is the token generated by student at position $h$, and $c_{i,h}^{r(i)}$ is the teacher-conditioning context attached to that token when the sample was graded.
Following the previous introduction of the distinction between $x$ and $c$, we now refine the sample-level notation from Sec.~\ref{sec:theory} by lifting $x_i$ to token-level replay prefixes and $c_i^t$ to token-level current teacher contexts, since replay-based comparison is performed only at aligned token positions. We likewise refine the stored prefix into token-level replay prefixes:
\begin{equation}
    x_{i,h} := (x_i, y_{i,<h}),
\end{equation}
so that the diagnostics and losses below are evaluated on aligned token positions and later aggregated back to the sample level.
Let $\mathcal{U}_i^{\mathrm{roll}},\mathcal{U}_i^{\mathrm{sup}}\subseteq\{1,\dots, H_i\}$ denote token positions for which replay yields valid and comparable conditioning. 
We assume that the aligned sets are nonempty whenever a sample is used to compute diagnostics or optimization losses. 
For each position, we use the replayed token prefix $x_{i,h}$ to compute token-level diagnostics:
\begin{equation}
\label{eq:kl-diagnostics}
\left\{
\begin{aligned}
D_i^{\mathrm{roll}} &=
\frac{1}{|\mathcal{U}_i^{\mathrm{roll}}|}
\sum_{h\in \mathcal{U}_i^{\mathrm{roll}}}
\mathrm{KL}\!\left(
    \pi_{\theta}^t(\cdot\mid x_{i,h})\,\|\,\pi_{\theta}^{r(i)}(\cdot\mid x_{i,h})
\right),
\\
D_i^{\mathrm{sup}} &=
\frac{1}{|\mathcal{U}_i^{\mathrm{sup}}|}
\sum_{h\in \mathcal{U}_i^{\mathrm{sup}}}
\mathrm{KL}\!\left(
    \pi_{\mathrm{teacher}}(\cdot\mid c_{i,h}^{t})\,\|\,\pi_{\mathrm{teacher}}(\cdot\mid c_{i,h}^{r(i)})
\right).
\end{aligned}
\right.
\end{equation}
Moreover, each sample carries an update lag $\tau_i = t-r(i)\ge 0$, providing a coarse-grained online proxy (see Appendix~\ref{appendix-sec:lag-bound}) for global freshness: older samples are generally less reliable and should eventually be discarded. 
The replay-based diagnostics $D_i^{\mathrm{roll}}$ and $D_i^{\mathrm{sup}}$ therefore complement this coarse age signal with finer-grained sample-specific drift information. At the same time, the inverse-age factor $(\tau_i+1)^{-1}$ remains a cheap operational proxy for the budget over which rollout drift may have accumulated. 
We therefore describe the sample-level freshness of $i$ at $t$ with
\begin{equation}
\label{eq:freshness}
    f_i = \frac{1}{\tau_i+1} \mathrm{exp}(-\widetilde{\Delta}_i^t), \hspace{0.5em} \mathrm{where} \hspace{0.5em} \widetilde{\Delta}_i^t =
    \alpha \sqrt{D_i^{\mathrm{roll}}} + \beta \sqrt{D_i^{\mathrm{sup}}} \approx \Delta^t_i.
\end{equation}
Coefficients $\alpha,\beta \ge 0$ map the two diagnostics onto a common operational scale. Importantly, the approximation $\widetilde{\Delta}_i^t$ to the discrepancy in Proposition~\ref{prop:async-discrepancy} should be interpreted as a monotone surrogate that ranks samples by their predicted contribution to stale-objective bias, rather than as an exact estimator of that bias. 
The $\mathrm{exp}(\cdot)$ in freshness calculation thus smoothly and monotonically downweights high-staleness samples, reducing variance while preserving relative ordering.

\subsection{\texorpdfstring{$\boldsymbol{f}$-OPD}{f-OPD}}
Next, we provide a step-by-step derivation of \textit{f}-OPD with freshness-aware control, starting from the vanilla sample-level distillation loss:
\begin{equation}
\label{eq:sample-aggregated-loss}
    \ell_i^t =
    \frac{1}{|\mathcal{U}_i^{\mathrm{sup}}|}
    \sum_{h\in \mathcal{U}_i^{\mathrm{sup}}}
    \ell\bigl(\pi_{\theta}^t(\cdot\mid x_{i,h}), \pi_{\mathrm{teacher}}(\cdot\mid c_{i,h}^{t})\bigr).
\end{equation}
Building on the freshness score defined in Eq.~\ref{eq:freshness}, we first introduce \textbf{\emph{freshness weighting}} by scaling each sample's contribution to the distillation loss according to its fidelity to the on-policy objective:
\begin{equation}
    \mathcal{L}^t_{\mathrm{distill}} = 
    \mathbb{E}_{i\sim \mathcal{B}^t}\Bigl[ \sigma(f_i - \xi) \hspace{0.2em} \ell_i^t \Bigr],
\end{equation}
where $\sigma(\cdot)$ is the ReLU function, so that a sample with the freshness score below the threshold $\xi$ will be excluded from optimization.
Weighting is a first layer of defense: it reduces further deviation caused by over-optimizing on stale samples. However, it does not explicitly constrain how far the current policy can deviate from the rollout policy associated with surviving samples, nor can it repair a mismatch already present in an active buffer that has become globally stale.
Moreover, replay-based diagnostics are meaningful only at token positions where the stale and current trajectories remain alignable. When the active buffer has low mean freshness, or when the aligned-position coverage ratios $|\mathcal{U}_i^{\mathrm{roll}}|/H_i$ and $|\mathcal{U}_i^{\mathrm{sup}}|/H_i$ become too small, the buffer no longer provides sufficiently reliable supervision or drift estimates.

To address these limitations, we introduce two additional mechanisms. 
On the one hand, to enforce local consistency between the current and rollout policy, we incorporate \textbf{\emph{rollout-anchored regularization}} into the objective. 
Specifically, the regularization term $\mathcal{R}_i^{\mathrm{anchor}}=D_i^{\mathrm{roll}}$ uses the rollout-drift diagnostic from Eq.~\ref{eq:kl-diagnostics} as an observable surrogate for the stale-occupancy mismatch analyzed in Sec.~\ref{sec:diagnostics}. Notably, $\mathcal{R}^{\mathrm{anchor}}_i$ is also modulated by freshness weighting, since corrective regularization becomes less effective as sample fidelity deteriorates.
On the other hand, we introduce \textbf{\emph{adaptive refresh}} to monitor the aggregate freshness score and alignment statistics of the active buffer:
\begin{equation}
\bar{f}^{t} = \frac{1}{|\mathcal{B}^t|}\sum_{i\in\mathcal{B}^t} f_i, \quad
\bar{\mathcal{M}}^{t,\mathrm{roll}} = \frac{1}{|\mathcal{B}^t|}\sum_{i\in\mathcal{B}^t} \frac{|\mathcal{U}_i^{\mathrm{roll}}|}{H_i}, \quad
\bar{\mathcal{M}}^{t,\mathrm{sup}} = \frac{1}{|\mathcal{B}^t|}\sum_{i\in\mathcal{B}^t} \frac{|\mathcal{U}_i^{\mathrm{sup}}|}{H_i}.
\end{equation}
The buffer $\mathcal{B}^t$ is refreshed whenever it no longer satisfies the required freshness or alignment conditions, either due to degraded sample-mean freshness or insufficient aligned-support coverage:
\begin{equation}
    \bar{f}^{t} \le \kappa_{f},
    \qquad \text{or} \qquad
    \bar{\mathcal{M}}^{t,\mathrm{roll}} < \kappa_{\mathrm{roll}},
    \qquad \text{or} \qquad
    \bar{\mathcal{M}}^{t,\mathrm{sup}} < \kappa_{\mathrm{sup}},
\end{equation}
where $\kappa_f, \kappa_{\mathrm{roll}}, \kappa_{\mathrm{sup}}$ denote the corresponding thresholds. 

Combining the above three mechanisms, we propose ${f}$-OPD with a freshness-aware objective:
\begin{equation}
    \mathcal{J}(\pi_{\theta}^t) =
    \mathbb{E}_{i\sim \mathcal{C}(\mathcal{B}^t)}\Bigl[
    \sigma(f_i - \xi) \hspace{0.2em} \bigl(\ell_i^t + \lambda\, \mathcal{R}_i^{\mathrm{anchor}}\bigr)
    \Bigr],
\end{equation}
where $\lambda$ adjusts the strength of rollout-anchored regularization, and $\mathcal{C}(\mathcal{B}^t)$ denotes the refresh mechanism of renewing the active buffer to restore a more reliable supervision stream and more informative drift diagnostics for ${f}$-OPD.
We provide pseudocode in Appendix Algorithm~\ref{alg:fopd}, which summarizes the complete procedure in an implementation-oriented form.

\section{Experiments}
\begin{figure*}[t]
    \centering
    \includegraphics[width=\textwidth]{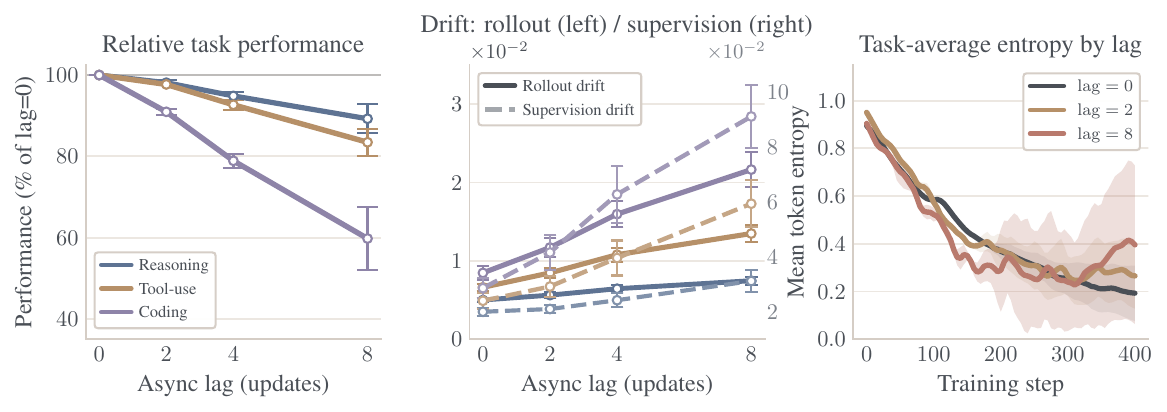}
    \caption{Failure modes of vanilla OPD under increasing policy update lag. (a) Relative task performance degradation with lags. (b) Drift diagnostics: rollout drift on the left axis, and supervision drift on the right axis, both increasing with lag. (c) Token-level entropy dynamics, averaged uniformly over three tasks, for lag~0, 2, and 8 across training. Markers and curves show five-seed means; error bars and shaded bands denote a standard deviation.}
    \label{fig:lag-diagnostics}
\end{figure*}

\begin{figure*}[b]
    \centering
    \includegraphics[width=0.98\textwidth]{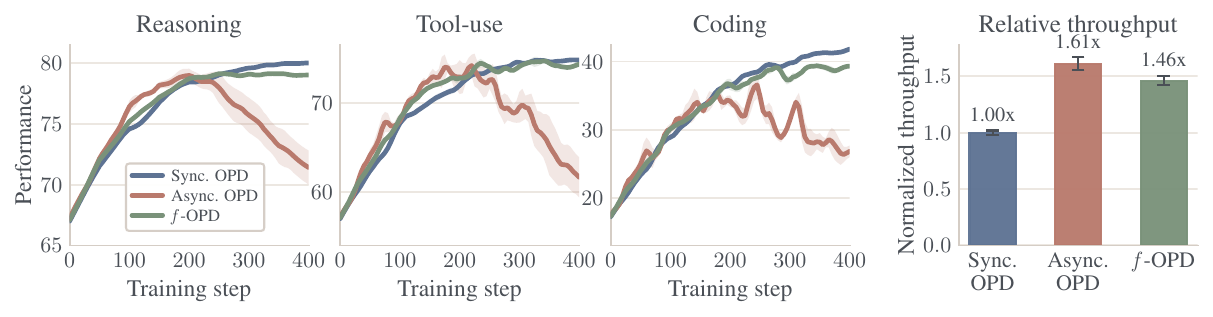}
    \caption{(a–c) Training dynamics across tasks for synchronous OPD, asynchronous OPD, and \textit{f}-OPD. (d) Relative training throughput across tasks.
    }
    \label{fig:task-dynamics-grid}
\end{figure*}

Our experiments highlight the challenge of balancing performance and efficiency in long-horizon OPD training, and demonstrate how the proposed ${f}$-OPD effectively navigates this trade-off with freshness-aware control.

\subsection{Settings}
We summarize the key experimental settings from the following three perspectives. 

\textbf{Protocols.} 
We follow the student-teacher distillation setup introduced by Thinking Machines Lab~\cite{lu2025onpolicy}. 
For model selection across tasks, we use Qwen family models as student models, trained under frozen supervision from larger Qwen teacher models. Training is conducted under both synchronous and asynchronous pipelines.

\textbf{Task taxonomies.} 
We consider three task families with increasing interaction horizon: (i) short-horizon reasoning, where supervision is effectively single-shot; (ii) medium-horizon tool-use, where intermediate actions influence subsequent observations and supervision; and (iii) long-horizon coding agent, where errors compound across iterative edits and executions. DAPO-Math-17K~\cite{bytedtsinghua2025dapomath17k} ($\sim$17K samples) is adopted as the underlying prompt distribution for both short-horizon reasoning and medium-horizon tool-use. In the tool-use setting, those problems are wrapped in a deterministic ReTool-style code-interpreter harness~\cite{feng2025retool}. For long-horizon coding-agent tasks, we use mini-SWE-agent~\cite{yang2024sweagent} as the coding-agent scaffold on real repository-level software-engineering tasks.

\textbf{Evaluations.} 
We evaluate short-horizon reasoning on MATH500~\cite{hendrycks2021measuring}, tool-use on a held-out olympiad-style set of 1,000 problems executed in the same harness, and coding agents on SWE-bench Verified~\cite{jimenez2024swebench}. For short-horizon reasoning, we report \emph{Avg@1 accuracy}. For tool-use, we report \emph{Avg@4 accuracy} under the ReTool-style harness, together with \emph{invalid tool rate} (the fraction of episodes containing at least one invalid code-interpreter invocation) and \emph{peak-final drop} (the drop in resolve rate from the peak checkpoint to the final checkpoint). For coding agents, we report the SWE-bench Verified \emph{resolve rate} under a single-run, single-patch evaluation protocol, along with repeated-loop, premature-stop, and \emph{post-patch regression} rates (the performance of the patch drops from peak to final training). 
\textcolor{black}{Beyond model performance, we measure end-to-end training throughput as the number of completed training samples per wall-clock hour to quantify system efficiency. Throughput is normalized to vanilla synchronous OPD, which serves as the baseline (i.e., Sync OPD $=1.00\times$). }
All reported metrics are averaged over five random seeds. 
More details about the experimental setting and implementation are provided in Appendix~\ref{append-sec:setting}.

\subsection{Failure Analysis of Vanilla OPD}

Before evaluating the effectiveness of ${f}$-OPD, we first conduct a systematic failure analysis of vanilla OPD. Specifically, we investigate two questions: (i) \emph{how does increasing policy lag affect task performance under OPD training?} and (ii) \emph{which factors are responsible for this degradation under increasing policy lag?}
To answer these questions, we evaluate four OPD variants across three tasks with manually controlled policy update lags of 0, 2, 4, and 8, spanning conditions from fully on-policy optimization to increasing asynchrony.

As shown in Figure~\ref{fig:lag-diagnostics}(a), all three tasks exhibit a clear and monotonic performance degradation as policy lag increases, confirming that policy staleness is inherently harmful to asynchronous OPD fidelity. Moreover, this degradation becomes substantially more severe as the interaction horizon grows. Particularly for the coding-agent task, performance drops by roughly 40\% under a lag of 8.
Figure~\ref{fig:lag-diagnostics}(b) further reports the KL-based diagnostics of rollout and supervision drift. Both drifts grow systematically with policy lag and interaction horizon, and supervision drift scales more steeply with lag than rollout drift, highlighting the importance of controlling teacher-context staleness in long-horizon OPD.
Interestingly, Figure~\ref{fig:lag-diagnostics}(c) shows that, unlike synchronous training, asynchronous OPD exhibits progressively less stable task-averaged entropy dynamics as lag increases. The highest-lag setting ($\mathrm{lag}=8$, red) develops the widest late-training uncertainty band, indicating that some runs retain high-entropy exploration whereas others move toward partial collapse. This suggests that policy drift under asynchronous execution not only amplifies objective mismatch but also destabilizes generation dynamics, making entropy a practical early-warning signal for training instability.

\begin{table}[t]
\centering
\caption{Comparison of \textit{f}-OPD against representative post-training methods on the coding-agent task~\cite{jimenez2024swebench}. Relative throughput is normalized to synchronous OPD. }
\label{tab:systems-tradeoff}

\small
\setlength{\tabcolsep}{4pt}

\begin{tabular}{@{}lccccc@{}}
\toprule
\textbf{Method}
& \textbf{Scale}
& \textbf{Throughput ↑}
& \textbf{Resolve ↑}
& \textbf{Post-patch reg. ↓}
& \textbf{Collapse ratio ↓} \\
\midrule

\multicolumn{6}{c}{\textit{Supervised Fine-tuning}} \\
\midrule
Vanilla SFT & 8B & $2.18\times$ & 39.1 & --- & 0/5 \\
SWE-Gym$^*$~\citep{pan2025training} & 32B & --- & 20.6 & --- & --- \\

\midrule
\multicolumn{6}{c}{\textit{Reinforcement Learning}} \\
\midrule
Vanilla GRPO & 8B & $0.92\times$ & 40.2 & --- & --- \\
SA-SWE$^*$~\citep{cao2025skyrlagent} & 32B & --- & 39.4 & --- & --- \\
DeepSWE$^*$~\cite{luo2025deepswe} & 32B & --- & 36.4 & --- & --- \\

\midrule
\multicolumn{6}{c}{\textit{On-Policy Distillation}} \\
\midrule
Qwen3-Coder$^{\dagger}$~\cite{cao2026qwen3codernext} & 30B & --- & 51.6 & --- & --- \\
Sync. OPD & 8B & $1.00\times$ & 41.8 & 1.6 & 0/5 \\
Async. OPD & 8B & $1.61\times$ & 26.8 & 7.4 & 3/5 \\
Async. OPD + KL & 8B & $1.60\times$ & 28.2 & 7.6 & 3/5 \\
Async. OPD + clipped IW & 8B & $1.60\times$ & 29.9 & 7.2 & 2/5 \\
Async. OPD + hard refresh & 8B & $1.15\times$ & 35.1 & 6.6 & 1/5 \\
Async. OPD + lag-only weighting & 8B & $1.54\times$ & 34.9 & 7.6 & 2/5 \\

\rowcolor{green!8}
\textbf{\textit{f}-OPD (ours)} & 8B & $\mathbf{1.46\times}$ & $\mathbf{39.4}$ & 2.6 & 0/5 \\

\bottomrule
\end{tabular}
\parbox{\linewidth}{\footnotesize
$^{*}$ Results reported from publicly available sources. 
\quad
$-$ Results not available.
\quad
$^{\dagger}$ Teacher model used for OPD. 
}
\end{table}


\subsection{Main Results of \texorpdfstring{$\boldsymbol{f}$-OPD}{f-OPD}}

Our main results are organized into two parts. We first benchmark \textit{f}-OPD against synchronous and asynchronous OPD baselines across the three tasks, measuring task quality (evaluation accuracy along training) and relative system throughput (completed training samples per wall-clock hour, normalized to synchronous OPD). Fig.~\ref{fig:task-dynamics-grid} shows that $f$-OPD inherits the task quality of synchronous OPD and much of the throughput advantage of asynchronous OPD simultaneously, recovering the on-policy fidelity that vanilla asynchronous mode sacrifices without paying the full throughput cost of synchronous execution. 
Additionally, as shown in Appendix Figure~\ref{fig:appendix-entropy-lag-compare}, $f$-OPD maintains entropy dynamics comparable to synchronous training, demonstrating the stability brought by freshness-aware control. 


To assess \textit{f}-OPD on the most demanding setting, we compare it against a broad spectrum of post-training methods on the long-horizon coding-agent task: SFT, RL, and two asynchronous OPD baselines with simplified freshness-aware control—one applying hard refresh to buffered samples, the other scoring sample freshness only by policy update lag (see Appendix~\ref{append-sec:Compared_methods} for details). The methods we reproduced share the same base model (Qwen3-8B). Results are reported in Table~\ref{tab:systems-tradeoff}. SFT achieves the highest throughput but is bottlenecked by its reliance on high-quality offline trajectories, which are costly to scale in long-horizon agentic settings. RL is on-policy by design but suffers from low throughput due to long-tailed rollout latency. Synchronous OPD achieves the highest task quality among on-policy methods (41.8 resolve rate) but inherits the throughput penalty of synchronous execution. Vanilla asynchronous OPD delivers a $1.61\times$ throughput gain over synchronous OPD, but task quality collapses to a 26.8 resolve with a 60\% collapse rate across runs. \textcolor{black}{To examine the limitations of directly transferring trust-region methods from RL to OPD, we incorporate KL- and importance-weighting (IW)-based approaches into asynchronous OPD. Both yield only modest reductions in the collapse rate because the drift inherent in teacher supervision is overlooked.} Additionally, simplified freshness-aware variants (hard refresh \& lag-only weighting) also partially mitigate the degradation and remain well below the synchronous ceiling.
In contrast, \textit{f}-OPD attains a $1.46\times$  throughput gain while preserving 94\% synchronous resolve and incurring only a 2.6-point post-patch regression, with no observed collapse. These results position \textit{f}-OPD as a method that closes the performance–efficiency gap left open by both general-purpose post-training paradigms and vanilla asynchronous OPD variants.

\subsection{Generalization and Compatibility}

\textcolor{black}{In addition to the effectiveness on tasks with different horizons, we evaluate the generalizability of \textit{f}-OPD across model families, distillation settings, and objective variants.}

\textcolor{black}{First, we consider two non-Qwen student-teacher pairs: Llama-3.1-8B-Instruct/Llama-3.1-70B-Instruct, representing capacity-gap distillation, and DeepSeekMath-7B-Base/DeepSeekMath-7B-RL, representing same-scale capability transfer. As shown in Appendix  Table~\ref{tab:cross-family}, \textit{f}-OPD consistently recovers the performance degradation induced by asynchronous execution, achieving results comparable to synchronous OPD across three mathematical reasoning benchmarks. These results demonstrate its generalizability beyond the Qwen family and across distinct distillation settings.
On the other hand, freshness-aware control is orthogonal to the underlying distillation objective. We therefore integrate \textit{f}-OPD with two alternative OPD methods, EOPD~\cite{jin2026entropy} and SOD~\cite{zhong2026sod}, to test the compatibility. As shown in Appendix   Table~\ref{tab:opd-variants}, \textit{f}-OPD recovers most of the accuracy lost to asynchronous execution while incurring less than 10\% throughput overhead over the corresponding asynchronous baselines. For example, on AIME24, it improves EOPD from 21.2 to 32.1, approaching the synchronous performance of 35.8.
}

\subsection{Ablation Study}

\begin{table}[t]
\centering
\small
\setlength{\tabcolsep}{3.5pt}

\resizebox{\linewidth}{!}{%
\begin{tabular}{@{}lcc@{\hspace{8pt}}ccc@{}}
\toprule
& \multicolumn{2}{c}{\textbf{Tool-use}}
& \multicolumn{3}{c}{\textbf{Coding agent}} \\
\cmidrule(lr){2-3}\cmidrule(lr){4-6}
\textbf{Method}
& \textbf{Acc.$_{\mathrm{avg@4}}$}
& \textbf{Peak-final drop}
& \textbf{Resolve}
& \textbf{Post-patch regression}
& \textbf{Collapse ratio} \\
\midrule
\textbf{Synchronous OPD}
& 74.8 & 2.7 & 41.8 & 1.6 & 0/5 \\
\midrule
\multicolumn{6}{c}{\textbf{\textit{Incremental composition}}} \\
\midrule

\textbf{Asynchronous OPD}
& 61.6 & 10.4 & 26.8 & 7.4 & 3/5 \\
+ Freshness weighting
& 70.8 & 4.3 & 36.5 & 4.4 & 1/5 \\
+ Rollout-anchored regularization
& 73.2 & 3.4 & 38.1 & 3.2 & 1/5 \\
+ Adaptive refresh (\textit{f}-OPD)
& 74.4 & 3.1 & 39.4 & 2.6 & 0/5 \\

\midrule

\multicolumn{6}{c}{\textbf{\textit{Leave-one-out}}} \\
\midrule
\textbf{\textit{f}-OPD}
& 74.4 & 3.1 & 39.4 & 2.6 & 0/5 \\
w/o Freshness weighting
& 69.8 & 3.6 & 37.4 & 3.5 & 3/5 \\
w/o Rollout-anchored regularization
& 73.1 & 3.2 & 37.7 & 3.4 & 1/5 \\
w/o Adaptive refresh
& 73.2 & 3.4 & 38.1 & 3.2 & 1/5 \\

\bottomrule
\end{tabular}%
}
\caption{Ablation of the three freshness-aware control mechanisms in
\textit{f}-OPD on tool-use and coding-agent tasks. The upper block
incrementally adds each mechanism, while the lower block removes each
component from the full \textit{f}-OPD.}
\label{tab:ablations}
\end{table}

The effectiveness of \textit{f}-OPD stems from the synergy of three principled mechanisms: \emph{freshness weighting}, \emph{rollout-anchored regularization}, and \emph{adaptive refresh}. We conduct two complementary ablation studies to assess their individual contributions. 

\textcolor{black}{First, we incrementally introduce each component on top of vanilla asynchronous OPD and evaluate their cumulative effects on the tool-use and coding-agent tasks (Table~\ref{tab:ablations}). Specifically, freshness weighting recovers most of the lost task quality (+9.2 avg@4 accuracy, +9.7 coding resolve). It reduces the collapse ratio to 1/5, confirming that down-weighting stale samples directly counters both rollout and supervision drift. Rollout-anchored regularization further closes the gap by stabilizing per-update drift (+2.4 avg@4 accuracy, +1.6 coding resolve), while adaptive refresh delivers the final gains and eliminates collapse. 
Second, we conduct a leave-one-out ablation by removing each mechanism individually from the full \textit{f}-OPD. The results further highlight the critical role of \emph{freshness weighting}: removing it leads to the largest performance degradation, reducing tool-use accuracy from 74.4 to 69.8 and coding-agent Resolve from 39.4 to 37.4, while the collapse ratio increases from 0/5 to 3/5, matching that of vanilla asynchronous OPD. In contrast, removing either \emph{rollout-anchored regularization} or \emph{adaptive refresh} causes only a modest increase in the collapse ratio (from 0/5 to 1/5), while still resulting in consistent degradation across the performance metrics. These suggest that freshness weighting is particularly important for maintaining training stability, while the other two mechanisms provide complementary gains in overall performance.}

\textcolor{black}{Alternatively, we conduct a hyperparameter sensitivity analysis (details in Appendix~\ref{app:ablation-sensitivity}) by grid-searching over the drift coefficients ($\alpha, \beta$) and refresh threshold ($\kappa$). While the best-performing configuration benefits from hyperparameter tuning, \textit{f}-OPD maintains stable performance across a relatively broad range of values, demonstrating its robustness to hyperparameter choices.}



\section{Conclusion}
\label{sec:conclusion}

In this paper, we systematically analyze the performance–efficiency trade-off of vanilla OPD under both synchronous and asynchronous execution. We identify the principal source of freshness mismatch in OPD from both the student rollout and the teacher context, and propose \textit{f}-OPD to guide asynchronous optimization through freshness-aware control. Its effectiveness is established across tasks of varying interaction horizons, with the most pronounced gains on long-horizon coding agents. Looking ahead, we aim to scale \textit{f}-OPD to larger foundation models and a broader spectrum of general-purpose agentic tasks.

\begin{ack}
This work was partially supported by the Research Grants Council of the Hong Kong Special Administrative Region, China (Grant Nos.~PolyU25216423 and PolyU15217424), and The Hong Kong Polytechnic University (Project Nos.~P0056503 and P0053699).
\end{ack}




{\small
\bibliographystyle{unsrtnat}
\bibliography{refs/references}

@article{grattafiori2024llama,
  title={The llama 3 herd of models},
  author={Grattafiori, Aaron and Dubey, Abhimanyu and Jauhri, Abhinav and Pandey, Abhinav and Kadian, Abhishek and Al-Dahle, Ahmad and Letman, Aiesha and Mathur, Akhil and Schelten, Alan and Vaughan, Alex and others},
  journal={arXiv preprint arXiv:2407.21783},
  year={2024}
}

@inproceedings{mirzadeh2025gsm,
  title={Gsm-symbolic: Understanding the limitations of mathematical reasoning in large language models},
  author={Mirzadeh, Iman and Alizadeh-Vahid, Keivan and Shahrokhi, Hooman and Tuzel, Oncel and Bengio, Samy and Farajtabar, Mehrdad},
  booktitle={International Conference on Learning Representations},
  volume={2025},
  pages={94743--94765},
  year={2025}
}

@article{zhong2026sod,
  title={Sod: Step-wise on-policy distillation for small language model agents},
  author={Zhong, Qiyong and Zheng, Mao and Song, Mingyang and Lin, Xin and Sun, Jie and Jiang, Houcheng and Wang, Xiang and Fang, Junfeng},
  journal={arXiv preprint arXiv:2605.07725},
  year={2026}
}

@article{williams1992simple,
  title={Simple Statistical Gradient-Following Algorithms for Connectionist Reinforcement Learning},
  author={Williams, Ronald J.},
  journal={Machine Learning},
  volume={8},
  number={3-4},
  pages={229--256},
  year={1992},
  doi={10.1007/BF00992696}
}

@article{li2026rethinking,
  title={Rethinking On-Policy Distillation of Large Language Models: Phenomenology, Mechanism, and Recipe},
  author={Li, Yaxuan and Zuo, Yuxin and He, Bingxiang and Zhang, Jinqian and Xiao, Chaojun and Qian, Cheng and Yu, Tianyu and Gao, Huan-ang and Yang, Wenkai and Liu, Zhiyuan and others},
  journal={arXiv preprint arXiv:2604.13016},
  year={2026}
}

@article{bengio2015scheduled,
  title={Scheduled sampling for sequence prediction with recurrent neural networks},
  author={Bengio, Samy and Vinyals, Oriol and Jaitly, Navdeep and Shazeer, Noam},
  journal={Advances in neural information processing systems},
  volume={28},
  year={2015}
}

@inproceedings{ross2011reduction,
      title={A Reduction of Imitation Learning and Structured Prediction to No-Regret Online Learning},
      author={Stephane Ross and Geoffrey J. Gordon and J. Andrew Bagnell},
      booktitle={Proceedings of the Fourteenth International Conference on Artificial Intelligence and Statistics},
      year={2011},
}

@inproceedings{schulman2015trust,
      title={Trust Region Policy Optimization},
      author={John Schulman and Sergey Levine and Philipp Moritz and Michael I. Jordan and Pieter Abbeel},
      booktitle={Proceedings of the 32nd International Conference on Machine Learning},
      year={2015},
}

@inproceedings{zhang2016staleness,
      title={Staleness-aware {Async-SGD} for Distributed Deep Learning},
      author={Wei Zhang and Suyog Gupta and Xiangru Lian and Ji Liu},
      booktitle={Proceedings of the Twenty-Fifth International Joint Conference on Artificial Intelligence},
      year={2016},
}

@inproceedings{mnih2016asynchronous,
      title={Asynchronous Methods for Deep Reinforcement Learning},
      author={Volodymyr Mnih and Adrià Puigdomènech Badia and Mehdi Mirza and Alex Graves and Timothy P. Lillicrap and Tim Harley and David Silver and Koray Kavukcuoglu},
      booktitle={Proceedings of the 33rd International Conference on Machine Learning},
      year={2016},
}

@article{schulman2017proximal,
      title={Proximal Policy Optimization Algorithms},
      author={John Schulman and Filip Wolski and Prafulla Dhariwal and Alec Radford and Oleg Klimov},
      journal={arXiv preprint arXiv:1707.06347},
      year={2017},
}

@article{hilton2021batch,
  title={Batch Size-Invariance for Policy Optimization},
  author={Hilton, Jacob and Cobbe, Karl and Schulman, John},
  journal={Advances in Neural Information Processing Systems},
  volume={35},
  pages={17086--17098},
  year={2022},
  note={Introduces decoupled PPO by separating the proximal policy for update control from the behavior policy for off-policy correction}
}

@misc{liu2025when,
  title={When Speed Kills Stability: Demystifying {RL} Collapse from the Training-Inference Mismatch},
  author={Liu, Jiacai and Li, Yingru and Fu, Yuqian and Wang, Jiawei and Liu, Qian and Shen, Yu},
  year={2025},
  month=sep,
  howpublished={\url{https://richardli.xyz/rl-collapse}},
  note={Research blog post, accessed 2026-05-06}
}

@article{li2025trust,
  title={Trust Region Masking for Long-Horizon {LLM} Reinforcement Learning},
  author={Li, Yingru and Liu, Jiacai and Xu, Jiawei and Tong, Yuxuan and Li, Ziniu and Liu, Qian and Wang, Baoxiang},
  journal={arXiv preprint arXiv:2512.23075},
  year={2025}
}

@article{zhang2026fast,
      title={Fast and Effective On-policy Distillation from Reasoning Prefixes},
      author={Dongxu Zhang and Zhichao Yang and Sepehr Janghorbani and Jun Han and Andrew Ressler II and Qian Qian and Gregory D. Lyng and Sanjit Singh Batra and Robert E. Tillman},
      journal={arXiv preprint arXiv:2602.15260},
      year={2026}
}

@article{song2026survey,
      title={A Survey of On-Policy Distillation for Large Language Models},
      author={Mingyang Song and Mao Zheng},
      journal={arXiv preprint arXiv:2604.00626},
      year={2026}
}

@article{ye2025blackbox,
      title={Black-Box On-Policy Distillation of Large Language Models},
      author={Tianzhu Ye and Li Dong and Zewen Chi and Xun Wu and Shaohan Huang and Furu Wei},
      journal={arXiv preprint arXiv:2511.10643},
      year={2025}
}

@article{jin2026entropy,
      title={Entropy-Aware On-Policy Distillation of Language Models},
      author={Woogyeol Jin and Taywon Min and Yongjin Yang and Swanand Ravindra Kadhe and Yi Zhou and Dennis Wei and Nathalie Baracaldo and Kimin Lee},
      journal={arXiv preprint arXiv:2603.07079},
      year={2026},
      note={Also available on OpenReview as SPOT 2026}
}

@inproceedings{espeholt2018impala,
      title={{IMPALA}: Scalable Distributed Deep-RL with Importance Weighted Actor-Learner Architectures},
      author={Lasse Espeholt and Hubert Soyer and Remi Munos and Karen Simonyan and Volodymir Mnih and Tom Ward and Yotam Doron and Vlad Firoiu and Tim Harley and Iain Dunning and Shane Legg and Koray Kavukcuoglu},
      booktitle={Proceedings of the 35th International Conference on Machine Learning},
      year={2018},
}

@inproceedings{harlap2019pipedream,
      title={{PipeDream}: Fast and Efficient Pipeline Parallel {DNN} Training},
      author={Aaron Harlap and Deepak Narayanan and Amar Phanishayee and Vivek Seshadri and Nikhil Devanur and Greg Ganger and Phil Gibbons},
      booktitle={Proceedings of the 27th ACM Symposium on Operating Systems Principles},
      year={2019},
}

@article{touvron2023llama2,
  title={{Llama 2}: Open Foundation and Fine-Tuned Chat Models},
  author={Hugo Touvron and Louis Martin and Kevin Stone and Peter Albert and Amjad Almahairi and Yasmine Babaei and Nikolay Bashlykov and Soumya Batra and Prajjwal Bhargava and Shruti Bhosale and Dan Bikel and Lukas Blecher and Cristian Canton Ferrer and Moya Chen and Guillem Cucurull and David Esiobu and Jude Fernandes and Jeremy Fu and Wenyin Fu and Brian Fuller and Cynthia Gao and Vedanuj Goswami and Naman Goyal and Anthony Hartshorn and Saghar Hosseini and Rui Hou and Hakan Inan and Marcin Kardas and Viktor Kerkez and Madian Khabsa and Isabel Kloumann and Artem Korenev and Punit Singh Koura and Marie-Anne Lachaux and Thibaut Lavril and Jenya Lee and Diana Liskovich and Yinghai Lu and Yuning Mao and Xavier Martinet and Todor Mihaylov and Pushkar Mishra and Igor Molybog and Yixin Nie and Andrew Poulton and Jeremy Reizenstein and Rashi Rungta and Kalyan Saladi and Alan Schelten and Ruan Silva and Eric Michael Smith and Ranjan Subramanian and Xiaoqing Ellen Tan and Binh Tang and Ross Taylor and Adina Williams and Jian Xiang Kuan and Puxin Xu and Zheng Yan and Iliyan Zarov and Yuchen Zhang and Angela Fan and Melanie Kambadur and Sharan Narang and Aurelien Rodriguez and Robert Stojnic and Sergey Edunov and Thomas Scialom},
  journal={arXiv preprint arXiv:2307.09288},
  year={2023}
}

@article{shao2024deepseekmath,
  title={DeepSeekMath: Pushing the limits of mathematical reasoning in open language models},
  author={Shao, Zhihong and Wang, Peiyi and Zhu, Qihao and Xu, Runxin and Song, Junxiao and Bi, Xiao and Zhang, Haowei and Zhang, Mingchuan and Li, YK and Wu, Yang and others},
  journal={arXiv preprint arXiv:2402.03300},
  year={2024}
}

@misc{deepseek2026deepseekv4,
  title={DeepSeek-V4: Towards Highly Efficient Million-Token Context Intelligence},
  author={{DeepSeek-AI}},
  year={2026},
  howpublished={\url{https://huggingface.co/deepseek-ai/DeepSeek-V4-Pro/blob/main/DeepSeek_V4.pdf}},
  note={technical report}
}

@misc{nemo-rl,
title = {NeMo RL: A Scalable and Efficient Post-Training Library},
howpublished = {\url{https://github.com/NVIDIA-NeMo/RL}},
author={{Nvidia}},
year = {2025},
note = {GitHub repository},
}

@article{yang2025qwen3,
  title={Qwen3 Technical Report},
  author={Yang, An and Li, Anfeng and Yang, Baosong and Zhang, Beichen and Hui, Binyuan and Zheng, Bo and Yu, Bowen and Gao, Chang and Huang, Chengen and Lv, Chenxu and others},
  journal={arXiv preprint arXiv:2505.09388},
  year={2025}
}

@article{xiao2026mimov2flash,
  title={MiMo-V2-Flash Technical Report},
  author={Xiao, Bangjun and Xia, Bingquan and Yang, Bo and Gao, Bofei and Shen, Bowen and Zhang, Chen and He, Chenhong and Lou, Chiheng and Luo, Fuli and Wang, Gang and others},
  journal={arXiv preprint arXiv:2601.02780},
  year={2026}
}

@article{zeng2026glm5,
  title={GLM-5: from Vibe Coding to Agentic Engineering},
  author={Zeng, Aohan and Lv, Xin and Hou, Zhenyu and Du, Zhengxiao and Zheng, Qinkai and Chen, Bin and Yin, Da and Ge, Chendi and Huang, Chenghua and Xie, Chengxing and others},
  journal={arXiv preprint arXiv:2602.15763},
  year={2026}
}

@misc{lu2025onpolicy,
  author = {Lu, Kevin and Thinking Machines Lab},
  title = {On-Policy Distillation},
  year = {2025},
  howpublished = {\url{https://thinkingmachines.ai/blog/on-policy-distillation/}},
  note = {Thinking Machines Lab blog post, published 2025-10-27, accessed 2026-05-06}
}

@misc{bytedtsinghua2025dapomath17k,
  author = {{BytedTsinghua-SIA}},
  title = {{DAPO-Math-17k}},
  year = {2025},
  howpublished = {\url{https://huggingface.co/datasets/BytedTsinghua-SIA/DAPO-Math-17k}},
  note = {Hugging Face dataset repository, accessed 2026-05-06}
}

@article{yu2025dapo,
  title={DAPO: An Open-Source LLM Reinforcement Learning System at Scale},
  author={Yu, Qiying and Zhang, Zheng and Zhu, Ruofei and Yuan, Yufeng and Zuo, Xiaochen and Yue, Yu and Dai, Weinan and Fan, Tiantian and Liu, Gaohong and Liu, Lingjun and others},
  journal={arXiv preprint arXiv:2503.14476},
  year={2025}
}

@article{cui2025entropy,
  title={The Entropy Mechanism of Reinforcement Learning for Reasoning Language Models},
  author={Cui, Ganqu and Zhang, Yuchen and Chen, Jiacheng and Yuan, Lifan and Wang, Zhi and Zuo, Yuxin and Li, Haozhan and Fan, Yuchen and Chen, Huayu and Chen, Weize and others},
  journal={arXiv preprint arXiv:2505.22617},
  year={2025}
}

@inproceedings{yang2024sweagent,
  title={{SWE}-agent: Agent-Computer Interfaces Enable Automated Software Engineering},
  author={Yang, John and Jimenez, Carlos E. and Wettig, Alexander and Lieret, Kilian and Yao, Shunyu and Narasimhan, Karthik R. and Press, Ofir},
  booktitle={Advances in Neural Information Processing Systems},
  year={2024},
  url={https://arxiv.org/abs/2405.15793},
  note={Recommended citation for mini-SWE-agent from the project repository}
}

@article{zhang2025relax,
  title={ReLaX: Reasoning with Latent Exploration for Large Reasoning Models},
  author={Zhang, Shimin and Chen, Xianwei and Shen, Yufan and Ye, Ziyuan and Wu, Jibin},
  journal={arXiv preprint arXiv:2512.07558},
  year={2025}
}

@article{feng2025retool,
  title={{ReTool}: Reinforcement Learning for Strategic Tool Use in {LLMs}},
  author={Jiazhan Feng and Shijue Huang and Xingwei Qu and Ge Zhang and Yujia Qin and Baoquan Zhong and Chengquan Jiang and Jinxin Chi and Wanjun Zhong},
  journal={arXiv preprint arXiv:2504.11536},
  year={2025}
}

@inproceedings{yang2025swesmith,
  title={{SWE-smith}: Scaling Data for Software Engineering Agents},
  author={Yang, John and Lieret, Kilian and Jimenez, Carlos E. and Wettig, Alexander and Khandpur, Kabir and Zhang, Yanzhe and Hui, Binyuan and Press, Ofir and Schmidt, Ludwig and Yang, Diyi},
  booktitle={Advances in Neural Information Processing Systems},
  year={2025},
  note={Datasets and Benchmarks Track Spotlight},
  url={https://openreview.net/forum?id=63iVrXc8cC}
}

@inproceedings{pan2025training,
  title={Training Software Engineering Agents and Verifiers with {SWE-Gym}},
  author={Jiayi Pan and Xingyao Wang and Graham Neubig and Navdeep Jaitly and Heng Ji and Alane Suhr and Yizhe Zhang},
  booktitle={Proceedings of the 42nd International Conference on Machine Learning},
  year={2025},
  note={arXiv:2412.21139}
}

@article{sheng2024hybridflow,
  title   = {HybridFlow: A Flexible and Efficient RLHF Framework},
  author  = {Guangming Sheng and Chi Zhang and Zilingfeng Ye and Xibin Wu and Wang Zhang and Ru Zhang and Yanghua Peng and Haibin Lin and Chuan Wu},
  year    = {2024},
  journal = {arXiv preprint arXiv:2409.19256}
}

@article{cao2025skyrlagent,
  title={SkyRL-Agent: Efficient RL Training for Multi-turn LLM Agent},
  author={Cao, Shiyi and Li, Dacheng and Zhao, Fangzhou and Yuan, Shuo and Hegde, Sumanth R. and Chen, Connor and Ruan, Charlie and Griggs, Tyler and Liu, Shu and Tang, Eric and Liaw, Richard and Moritz, Philipp and Zaharia, Matei and Gonzalez, Joseph E. and Stoica, Ion},
  journal={arXiv preprint arXiv:2511.16108},
  year={2025}
}

@article{cao2026qwen3codernext,
  title={Qwen3-Coder-Next Technical Report},
  author={Cao, Ruisheng and Chen, Mouxiang and Chen, Jiawei and Cui, Zeyu and Feng, Yunlong and Hui, Binyuan and Jing, Yuheng and Li, Kaixin and Li, Mingze and Lin, Junyang and others},
  journal={arXiv preprint arXiv:2603.00729},
  year={2026}
}

@misc{luo2025deepswe,
  author = {Luo, Michael and Jain, Naman and Singh, Jaskirat and Tan, Sijun and Cai, Colin and Venkat, Tarun and Roongta, Manan and Li, Li Erran and Popa, Raluca Ada and Sen, Koushik and Stoica, Ion and Patel, Ameen and Wu, Qingyang and Ariyak, Alpay and Zhu, Shang and Athiwaratkun, Ben and Zhang, Ce},
  title = {DeepSWE: Training a Fully Open-sourced, State-of-the-Art Coding Agent by Scaling RL},
  year = {2025},
  howpublished = {\url{https://pretty-radio-b75.notion.site/DeepSWE-Training-a-Fully-Open-sourced-State-of-the-Art-Coding-Agent-by-Scaling-RL-22281902c1468193aabbe9a8c59bbe33}},
  note = {Technical blog post, accessed 2026-05-06}
}

@misc{zhu2025slime,
  author = {Zhu, Zilin and Xie, Chengxing and Lv, Xin and slime Contributors},
  title = {slime: An LLM Post-Training Framework for RL Scaling},
  year = {2025},
  howpublished = {\url{https://github.com/THUDM/slime}},
  note = {GitHub repository, accessed 2026-05-06}
}

@article{sang2026crisp,
  title={CRISP: Compressed Reasoning via Iterative Self-Policy Distillation},
  author={Sang, Hejian and Xu, Yuanda and Zhou, Zhengze and He, Ran and Wang, Zhipeng and Sun, Jiachen},
  journal={arXiv preprint arXiv:2603.05433},
  year={2026}
}

@article{zhang2026opsdl,
  title={OPSDL: On-Policy Self-Distillation for Long-Context Language Models},
  author={Zhang, Xinsen and Ding, Zhenkai and Pan, Tianjun and Yang, Run and Kang, Chun and Xiong, Xue and Gu, Jingnan},
  journal={arXiv preprint arXiv:2604.17535},
  year={2026}
}

@article{hubotter2026reinforcement,
  title={Reinforcement Learning via Self-Distillation},
  author={H{\"u}botter, Jonas and L{\"u}beck, Frederike and Behric, Lejs and Baumann, Anton and Bagatella, Marco and Marta, Daniel and Hakimi, Ido and Shenfeld, Idan and Buening, Thomas Kleine and Guestrin, Carlos and others},
  journal={arXiv preprint arXiv:2601.20802},
  year={2026}
}

@article{zhao2026selfdistilledreasoner,
  title={Self-Distilled Reasoner: On-Policy Self-Distillation for Large Language Models},
  author={Zhao, Siyan and Xie, Zhihui and Liu, Mengchen and Huang, Jing and Pang, Guan and Chen, Feiyu and Grover, Aditya},
  journal={arXiv preprint arXiv:2601.18734},
  year={2026}
}

@inproceedings{hendrycks2021measuring,
      title={Measuring Mathematical Problem Solving with the {MATH} Dataset},
      author={Dan Hendrycks and Collin Burns and Saurav Kadavath and Akul Arora and Steven Basart and Eric Tang and Dawn Song and Jacob Steinhardt},
      booktitle={Proceedings of the Neural Information Processing Systems Track on Datasets and Benchmarks},
      year={2021},
}

@article{chen2021evaluating,
      title={Evaluating Large Language Models Trained on Code},
      author={Mark Chen and Jerry Tworek and Heewoo Jun and Qiming Yuan and Henrique Ponde de Oliveira Pinto and Jared Kaplan and Harri Edwards and Yuri Burda and Nicholas Joseph and Greg Brockman and Alex Ray and Raul Puri and Gretchen Krueger and Michael Petrov and Heidy Khlaaf and Girish Sastry and Pamela Mishkin and Brooke Chan and Scott Gray and Nick Ryder and Mikhail Pavlov and Alethea Power and Lukasz Kaiser and Mohammad Bavarian and Clemens Winter and Philippe Tillet and Felipe Petroski Such and Dave Cummings and Matthias Plappert and Fotios Chantzis and Elizabeth Barnes and Ariel Herbert-Voss and William Hebgen Guss and Alex Nichol and Alex Paino and Nikolas Tezak and Jie Tang and Igor Babuschkin and Suchir Balaji and Shantanu Jain and William Saunders and Christopher Hesse and Andrew N. Carr and Jan Leike and Josh Achiam and Vedant Misra and Evan Morikawa and Alec Radford and Matthew Knight and Miles Brundage and Mira Murati and Katie Mayer and Peter Welinder and Bob McGrew and Dario Amodei and Sam McCandlish and Ilya Sutskever and Wojciech Zaremba},
      journal={arXiv preprint arXiv:2107.03374},
      year={2021},
}

@article{austin2021program,
      title={Program Synthesis with Large Language Models},
      author={Jacob Austin and Augustus Odena and Maxwell Nye and Maarten Bosma and Henryk Michalewski and David Dohan and Ellen Jiang and Carrie Cai and Michael Terry and Quoc Le and Charles Sutton},
      journal={arXiv preprint arXiv:2108.07732},
      year={2021},
}

@inproceedings{gu2024minillm,
      title={{MiniLLM}: On-Policy Distillation of Large Language Models},
      author={Yuxian Gu and Li Dong and Furu Wei and Minlie Huang},
      booktitle={International Conference on Learning Representations},
      year={2024},
}

@inproceedings{agarwal2024onpolicy,
      title={On-Policy Distillation of Language Models: Learning from Self-Generated Mistakes},
      author={Rishabh Agarwal and Nino Vieillard and Yongchao Zhou and Piotr Stanczyk and Sabela Ramos and Matthieu Geist and Olivier Bachem},
      booktitle={International Conference on Learning Representations},
      year={2024},
}

@article{zhou2024webarena,
      title={{WebArena}: A Realistic Web Environment for Building Autonomous Agents},
      author={Shuyan Zhou and Frank F. Xu and Hao Zhu and Xuhui Zhou and Robert Lo and Abishek Sridhar and Xianyi Cheng and Tianyue Ou and Yonatan Bisk and Daniel Fried and Uri Alon and Graham Neubig},
      journal={arXiv preprint arXiv:2307.13854},
      year={2024},
}

@inproceedings{liu2024agentbench,
      title={{AgentBench}: Evaluating {LLMs} as Agents},
      author={Xiao Liu and Hao Yu and Hanchen Zhang and Yifan Xu and Xuanyu Lei and Hanyu Lai and Yu Gu and Hangliang Ding and Kaiwen Men and Kejuan Yang and Shudan Zhang and Xiang Deng and Aohan Zeng and Zhengxiao Du and Chenhui Zhang and Sheng Shen and Tianjun Zhang and Yu Su and Huan Sun and Minlie Huang and Yuxiao Dong and Jie Tang},
      booktitle={International Conference on Learning Representations},
      year={2024},
}

@inproceedings{jimenez2024swebench,
      title={{SWE-bench}: Can Language Models Resolve Real-World {GitHub} Issues?},
      author={Carlos E. Jimenez and John Yang and Alexander Wettig and Shunyu Yao and Kexin Pei and Ofir Press and Karthik Narasimhan},
      booktitle={International Conference on Learning Representations},
      year={2024},
}

@inproceedings{koh2024visualwebarena,
      title={{VisualWebArena}: Evaluating Multimodal Agents on Realistic Visual Web Tasks},
      author={Jing Yu Koh and Robert Lo and Lawrence Jang and Vikram Duvvur and Ming Chong Lim and Po-Yu Huang and Graham Neubig and Shuyan Zhou and Ruslan Salakhutdinov and Daniel Fried},
      booktitle={Proceedings of the 62nd Annual Meeting of the Association for Computational Linguistics},
      year={2024},
}

@article{yang2026learning,
      title={Learning beyond Teacher: Generalized On-Policy Distillation with Reward Extrapolation},
      author={Wenkai Yang and Weijie Liu and Ruobing Xie and Kai Yang and Saiyong Yang and Yankai Lin},
      journal={arXiv preprint arXiv:2602.12125},
      year={2026},
}

@article{ye2026onpolicy,
      title={On-Policy Context Distillation for Language Models},
      author={Tianzhu Ye and Li Dong and Xun Wu and Shaohan Huang and Furu Wei},
      journal={arXiv preprint arXiv:2602.12275},
      year={2026},
}

@article{ko2026scaling,
      title={Scaling Reasoning Efficiently via Relaxed On-Policy Distillation},
      author={Jongwoo Ko and Sara Abdali and Young Jin Kim and Tianyi Chen and Pashmina Cameron},
      journal={arXiv preprint arXiv:2603.11137},
      year={2026},
}

@article{fu2026revisiting,
      title={Revisiting On-Policy Distillation: Empirical Failure Modes and Simple Fixes},
      author={Yuqian Fu and Haohuan Huang and Kaiwen Jiang and Yuanheng Zhu and Dongbin Zhao},
      journal={arXiv preprint arXiv:2603.25562},
      year={2026},
}
}

\newpage

\appendix
\section{Additional Results for Theoretical Analysis}
\label{appendix-sec:theory_proof}
In this section, we provide the formal assumptions and proof sketches that support the theoretical analysis in Sec.~\ref{sec:theory} on objective discrepancy between synchronous and asynchronous OPD.

\subsection{Alternative Assumptions}

We reuse the notation from Sections~3 and~4. The current student policy at update $t$ is $\pi_{\theta}^t$, the rollout step attached to sample $i$ is $r(i)$, and the active stale buffer is $\mathcal{B}^t$. The ideal policy update objective is
\[
\mathcal{J}_{\mathrm{sync}}(\pi_{\theta}^t)
=
\mathbb{E}_{x\sim d^{t}}\bigl[\ell\bigl(\pi_{\theta}^t(\cdot\mid x), \pi_{\mathrm{teacher}}(\cdot\mid x)\bigr)\bigr],
\]
where $d^{t}$ denotes the mixture prefix distribution from Sec.~3. The stale-buffer objective is
\[
\mathcal{J}_{\mathrm{async}}(\pi_{\theta}^t)
=
\mathbb{E}_{i\sim\mathcal{B}^t}\bigl[\ell\bigl(\pi_{\theta}^t(\cdot\mid x_i), \pi_{\mathrm{teacher}}(\cdot\mid c_i^{r(i)})\bigr)\bigr].
\]
The discrepancy of interest is
\[
\Delta^t = \bigl|\mathcal{J}_{\mathrm{async}}(\pi_{\theta}^t)-\mathcal{J}_{\mathrm{sync}}(\pi_{\theta}^t)\bigr|.
\]

For the subsequent proof sketches, we first supplement the following assumptions.


\begin{assumption}[Total-variation stability of the loss]
\label{assump:tv-lipschitz}
On the truncated support from Assumption~\ref{assumption:regularity}, there exist constants $L_\pi,L_q\ge 0$ such that for any predictive distributions $p,p',q,q'$,
\[
|\ell(p,q)-\ell(p',q)| \le L_\pi\,\mathrm{TV}(p,p')
\qquad\text{and}\qquad
|\ell(p,q)-\ell(p,q')| \le L_q\,\mathrm{TV}(q,q').
\]
\end{assumption}

\begin{assumption}[Bounded one-step rollout drift on stale support]
\label{assump:step-drift}
For every update step $s$ and every stale training unit $x$ represented in the active buffer,
\[
\mathrm{TV}\!\left(\pi_{\theta}^{s+1}(\cdot\mid x),\pi_{\theta}^{s}(\cdot\mid x)\right) \le \varepsilon_s(x).
\]
\end{assumption}

Assumptions~\ref{assumption:regularity}, \ref{assump:tv-lipschitz}, and \ref{assump:step-drift} are stronger than the unconstrained LLM setting and are introduced only to make a discrepancy analysis finite and interpretable. The resulting bounds should therefore be read as conservative control justifications rather than a complete optimization theory for asynchronous OPD.

\subsection{Sketch of the Objective Discrepancy Decomposition}
\label{appendix-sec:proof-decomposition}
In this section, we sketch the argument behind Proposition~\ref{prop:async-discrepancy}. 
\begin{proof}
Recall the intermediate objective defined in Eq.~\ref{eq:intermediate}
\[
\widetilde{\mathcal{J}}_{\mathrm{async}}(\pi_{\theta}^t)
=
\mathbb{E}_{x\sim d^{\mathrm{stale},t}}\bigl[\ell\bigl(\pi_{\theta}^t(\cdot\mid x), \pi_{\mathrm{teacher}}(\cdot\mid x)\bigr)\bigr],
\]
where $d^{\mathrm{stale},t}$ denotes the mixture stale-prefix distribution induced by the active buffer. Then we derive
\[
\Delta^t
\le
\bigl|\mathcal{J}_{\mathrm{async}}(\pi_{\theta}^t)-\widetilde{\mathcal{J}}_{\mathrm{async}}(\pi_{\theta}^t)\bigr|
+
\bigl|\widetilde{\mathcal{J}}_{\mathrm{async}}(\pi_{\theta}^t)-\mathcal{J}_{\mathrm{sync}}(\pi_{\theta}^t)\bigr|.
\]
The first term is a pure supervision-mismatch quantity. By Assumption~\ref{assump:tv-lipschitz},
\[
\bigl|\mathcal{J}_{\mathrm{async}}(\pi_{\theta}^t)-\widetilde{\mathcal{J}}_{\mathrm{async}}(\pi_{\theta}^t)\bigr|
\le
L_q\,\mathbb{E}_{i\sim\mathcal{B}^t}\bigl[\mathrm{TV}\!\left(\pi_{\mathrm{teacher}}(\cdot\mid c_i^t), \pi_{\mathrm{teacher}}(\cdot\mid c_i^{r(i)})\right)\bigr].
\]
The second term depends only on replacing the stale distributional occupancy with the current occupancy, while keeping the current policy and labels fixed. Because the loss is bounded on the truncated support, standard change-of-measure arguments give
\[
\bigl|\widetilde{\mathcal{J}}_{\mathrm{async}}(\pi_{\theta}^t)-\mathcal{J}_{\mathrm{sync}}(\pi_{\theta}^t)\bigr|
\le
L_{\mathrm{occ}}\,\mathrm{TV}(d^{t},d^{\mathrm{stale},t}),
\]
for some constant $L_{\mathrm{occ}}$ depending only on the bounded support and loss normalization. Writing $d^{\mathrm{stale},t}=\mathbb{E}_{i\sim\mathcal{B}^t}[d^{r(i)}]$, convexity of total variation further gives
\[
\mathrm{TV}(d^{t},d^{\mathrm{stale},t})
\le
\mathbb{E}_{i\sim\mathcal{B}^t}\!\left[\mathrm{TV}(d^{t},d^{r(i)})\right].
\]
This is why the student-side discrepancy may also be interpreted as the buffer-average staleness induced by the sample-specific rollout steps. Combining the two inequalities yields Proposition~\ref{prop:async-discrepancy}.
\end{proof}

The consequence is important: the theoretical discrepancy is a two-term statement---rollout drift and supervision drift. 




\subsection{Relationship between Policy Update Lag and Rollout Drift}
\label{appendix-sec:lag-bound}
The following lemma shows that, under a mild step-wise smoothness assumption, the per-sample rollout drift can be upper-bounded by the number of update steps separating the rollout from the current student. This integer is essentially free to track during training.

\begin{lemma}[Lag bounds rollout drift]
\label{lem:lag-proxy}
Under Assumption~\ref{assump:step-drift}, for every buffered sample $i \in \mathcal{B}^t$,
\[
\delta_i^{\mathrm{rollout}}
\;\le\;
\sum_{s=r(i)}^{t-1} \varepsilon_s(x_i),
\]
where the right-hand side accumulates over exactly $\tau_i := t - r(i)$ steps.
\end{lemma}

\begin{proof}
By the triangle inequality for total variation,
\[
\delta_i^{\mathrm{rollout}}
=
\mathrm{TV}\!\left(\pi_{\theta}^t(\cdot\mid x_i),\,\pi_{\theta}^{r(i)}(\cdot\mid x_i)\right)
\le
\sum_{s=r(i)}^{t-1}
\mathrm{TV}\!\left(\pi_{\theta}^{s+1}(\cdot\mid x_i),\,\pi_{\theta}^{s}(\cdot\mid x_i)\right).
\]
Applying Assumption~\ref{assump:step-drift} to each summand yields the claim.
\end{proof}

Update lag is not a divergence per se, but a budget: it bounds rather than measures policy drift, with the bound becoming looser as per-step errors accumulate. Hence, $\tau_i$ is incorporated in our computation of the sample-level freshness score in Eq.~\ref{eq:freshness}.

\subsection{Rollout Drift Compounds with Horizon}
\label{appendix-sec:rollout_drift_horizon}
We further analyze how rollout drift scales with horizon length, which is particularly important for practical OPD deployment in long-horizon agentic scenarios.

\begin{proposition}[Horizon compounding of rollout staleness]
\label{prop:horizon-compounding}
Consider a finite-horizon setting with horizon $H$. Suppose that perturbing the policy on a prefix by total variation at most $\delta$ perturbs the next-step prefix distribution by at most $L \delta$. Then there exists a constant $K_H = O(HL)$ such that
\[
\mathrm{TV}\!\left(d^{t},\, d^{\mathrm{stale},t}\right)
\le
K_H\,\mathbb{E}_{i\sim\mathcal{B}^t}
\!\left[\delta_i^{\mathrm{roll}}\right]
\le
K_H\,\mathbb{E}_{i\sim\mathcal{B}^t}
\!\left[\sqrt{\tfrac{1}{2}\, D_i^{\mathrm{roll}}}\right].
\]
\end{proposition}

The main-text horizon proposition can be interpreted as a sequential perturbation bound on the mixture prefix distributions induced by a finite-horizon rollout. Let $d^{\pi}$ denote the resulting mixture distribution over prefixes under policy $\pi$. Suppose that replacing one policy by another changes the next-step prefix distribution by at most $L$ times the total-variation distance between the two policies on the current prefix. Then, by recursively unrolling the perturbation through the trajectory,
\[
\mathrm{TV}\!\left(d^{t},d^{\mathrm{stale},t}\right)
\le
K_H\,\mathbb{E}_{i\sim\mathcal{B}^t}\bigl[\delta_i^{\mathrm{roll}}\bigr]
\]
for some constant $K_H = O(HL)$ depending on horizon and the one-step sensitivity constant. This is the same compounding mechanism familiar from sequential prediction and imitation learning: small early action differences alter later prefixes, which in turn alter the observations, tool outputs, and teacher-conditioning contexts seen downstream.

In asynchronous OPD, this means that rollout drift and supervision drift need not be independent at long horizons; stale early actions can create the context drift that later enlarges $D_i^{\mathrm{sup}}$.

\subsection{Additional Analysis for Supervision Drift}
\label{appendix-sec:sup_drift_analysis}

The teacher-side staleness is defined in the sample-level shorthand from Sec.~\ref{sec:theory} by
\[
D_i^{\mathrm{sup}} = \mathrm{KL}\!\left(\pi_{\mathrm{teacher}}(\cdot\mid c_i^t)\,\|\,\pi_{\mathrm{teacher}}(\cdot\mid c_i^{r(i)})\right).
\]
In Sec.~\ref{sec:method}, this shorthand is instantiated by averaging token-level teacher comparisons over aligned positions. Here $\pi_{\mathrm{teacher}}(\cdot\mid c_i^{r(i)})$ is the teacher distribution actually cached and consumed during training. In contrast, $\pi_{\mathrm{teacher}}(\cdot\mid c_i^{t})$ is the teacher distribution obtained by relabeling the same training unit under the current context. The key point is that this term measures supervision drift, not teacher-parameter drift. It is enough for the effective teacher to change the context.

\begin{proposition}[Context mismatch perturbs the objective]
\label{prop:teacher-context}
Under Assumption~\ref{assump:tv-lipschitz}, the loss perturbation induced by stale teacher labels satisfies
\[
\begin{aligned}
&\bigl|\ell\bigl(\pi_{\theta}^t(\cdot\mid x_i), \pi_{\mathrm{teacher}}(\cdot\mid c_i^{t})\bigr) - \ell\bigl(\pi_{\theta}^t(\cdot\mid x_i), \pi_{\mathrm{teacher}}(\cdot\mid c_i^{r(i)})\bigr)\bigr| \\
&\le
L_q\,\mathrm{TV}\!\left(\pi_{\mathrm{teacher}}(\cdot\mid c_i^{t}), \pi_{\mathrm{teacher}}(\cdot\mid c_i^{r(i)})\right),
\end{aligned}
\]
and therefore the buffer-averaged supervision drift is bounded by a constant times the mean supervision-side total variation. By Pinsker, this is, in turn, controlled by the square root of the mean supervision KL.
\end{proposition}

This proposition is the formal justification for including $D_i^{\mathrm{sup}}$ in the sample-level freshness score. The quantity need not be large in every setting. Still, whenever teacher outputs depend on the mutable trajectory context, a stale label can shift the implemented objective away from the supervision attached to the current trajectory.

\begin{remark}[When the supervision gap is negligible]
If the teacher is fixed, the effective context is unchanged, and labeling is deterministic, then
\[
\pi_{\mathrm{teacher}}(\cdot\mid c_i^t) = \pi_{\mathrm{teacher}}(\cdot\mid c_i^{r(i)})
\qquad\Longrightarrow\qquad
D_i^{\mathrm{sup}} = 0.
\]
More generally, if the teacher distribution is stable under small context perturbations, then $D_i^{\mathrm{sup}}$ remains small. In those situations, the freshness framework naturally reduces to control for update lag and rollout drift.
\end{remark}



\begin{algorithm*}[t]
\small
\caption{${f}$-OPD training step. The algorithm first measures per-sample freshness from replay-based rollout and supervision diagnostics, then performs buffer-level refresh if aggregate freshness or aligned support degrades, and finally optimizes the freshness-weighted, rollout-anchored objective on the surviving samples.}
\label{alg:fopd}
\KwIn{active buffer $\mathcal{B}^t$; current student $\pi_\theta^t$; stale rollout checkpoints $\{\pi_\theta^{r(i)}\}$; teacher model $\pi_{\mathrm{teacher}}$; thresholds $\xi,\kappa_f,\kappa_{\mathrm{roll}},\kappa_{\mathrm{sup}}$; weights $\alpha,\beta,\lambda$}
\ForEach{$i \in \mathcal{B}^t$}{
replay sample $i$ on aligned token positions $\mathcal{U}_i^{\mathrm{roll}}, \mathcal{U}_i^{\mathrm{sup}}$\;
\If{$|\mathcal{U}_i^{\mathrm{roll}}|=0$ \textbf{or} $|\mathcal{U}_i^{\mathrm{sup}}|=0$}{
set $f_i=0$, $m_i^{\mathrm{roll}}=0$, and $m_i^{\mathrm{sup}}=0$\;
mark sample $i$ invalid for optimization and continue\;
}
compute rollout drift $D_i^{\mathrm{roll}}$ and supervision drift $D_i^{\mathrm{sup}}$ by Eq.~\ref{eq:kl-diagnostics}\;
compute lag $\tau_i = t-r(i)$ and freshness score $f_i = (\tau_i+1)^{-1}\exp(-\alpha \sqrt{D_i^{\mathrm{roll}}} - \beta \sqrt{D_i^{\mathrm{sup}}})$\;
compute alignment ratios $m_i^{\mathrm{roll}} = |\mathcal{U}_i^{\mathrm{roll}}| / H_i$ and $m_i^{\mathrm{sup}} = |\mathcal{U}_i^{\mathrm{sup}}| / H_i$\;
}
aggregate $\bar{f}^{t} = |\mathcal{B}^t|^{-1}\sum_{i\in\mathcal{B}^t} f_i$, $\bar{\mathcal{M}}^{t,\mathrm{roll}} = |\mathcal{B}^t|^{-1}\sum_i m_i^{\mathrm{roll}}$, and $\bar{\mathcal{M}}^{t,\mathrm{sup}} = |\mathcal{B}^t|^{-1}\sum_i m_i^{\mathrm{sup}}$\;
\If{$\bar{f}^{t} \le \kappa_f$ \textbf{or} $\bar{\mathcal{M}}^{t,\mathrm{roll}} < \kappa_{\mathrm{roll}}$ \textbf{or} $\bar{\mathcal{M}}^{t,\mathrm{sup}} < \kappa_{\mathrm{sup}}$}{
refresh $\mathcal{B}^t$ with newly rolled out and relabeled samples\;
recompute $\{D_i^{\mathrm{roll}}, D_i^{\mathrm{sup}}, f_i\}_{i\in\mathcal{B}^t}$ on the refreshed buffer\;
}
\ForEach{minibatch $\mathcal{M} \subset \mathcal{B}^t$}{
form freshness weights $w_i = \sigma(f_i-\xi)$ for $i \in \mathcal{M}$\;
form aggregated current-context loss $\ell_i^t$ by averaging $\ell\bigl(\pi_{\theta}^t(\cdot\mid x_{i,h}), \pi_{\mathrm{teacher}}(\cdot\mid c_{i,h}^{t})\bigr)$ over $h\in\mathcal{U}_i^{\mathrm{sup}}$\;
form anchored loss $\mathcal{L}_{\mathcal{M}} = \frac{1}{|\mathcal{M}|}\sum_{i\in\mathcal{M}} w_i\left[\ell_i^t + \lambda D_i^{\mathrm{roll}}\right]$\;
update student parameters $\theta \leftarrow \theta - \eta \nabla_\theta \mathcal{L}_{\mathcal{M}}$\;
}
\KwOut{updated student $\pi_\theta^{t+1}$ and refreshed metadata for the next asynchronous step}
\end{algorithm*}

\section{Additional Experimental Details}
\subsection{Settings}
\label{append-sec:setting}
\paragraph{Benchmark inventory and splits}
The paper is organized around three task families: short-horizon mathematical reasoning, medium-horizon tool-use reasoning, and long-horizon coding-agent issue resolution. The short-horizon benchmark uses MATH500 \citep{hendrycks2021measuring} together with a held-out 1{,}200-problem olympiad-style pool constructed from recent public contest problems; 200 problems are reserved as a pilot split for threshold selection, and 1{,}000 are used for the reported evaluation tables. The tool-use benchmark uses that same 1{,}000-problem report split, but each problem is wrapped in a deterministic code-interpreter environment following the ReTool formulation of interleaved reasoning and tool execution \citep{feng2025retool}. For coding, training trajectories are drawn from the mini-SWE agent, while the final coding evaluation uses 300 SWE-bench Verified issues \citep{jimenez2024swebench}; 50 issues form a pilot split, and 250 form the reported evaluation set. We reference Codex-style evaluation \citep{chen2021evaluating}, early program-synthesis benchmarks \citep{austin2021program}, AgentBench \citep{liu2024agentbench}, WebArena \citep{zhou2024webarena}, and VisualWebArena \citep{koh2024visualwebarena} only as neighboring evaluation settings rather than as the paper's primary evidence.

\paragraph{Details for the tool-use and coding-agent environment}
The tool-use environment follows ReTool \citep{feng2025retool} and exposes a single deterministic code-interpreter tool that can be interleaved with natural-language reasoning. Each tool invocation is a structured Python execution request; the harness returns the execution status along with stdout and stderr, and that observation is appended to the running dialogue state before the next student action. An invalid tool call is any malformed tool block, an unparsable execution payload, an unsupported invocation format, or a tool response that cannot be re-integrated into the active context. Each episode allows at most six tool calls and a maximum of ten total assistant turns.

The coding-agent setup adopts mini-swe-agent as the scaffold \citep{yang2025swesmith}, and the final benchmarked runs are evaluated on SWE-bench Verified instances \citep{jimenez2024swebench}. For each evaluation instance, the agent receives a repository snapshot before the human fix, the issue description, and the associated test-facing sandbox, then interacts with the repository through an edit-execute-observe loop within an isolated, containerized environment. We use the same internal edit-execute-observe scaffold training and evaluation, rather than switching to a separate leaderboard-specific scaffold; only the repository instances and the final grading harness differ across the two stages. The final patch is graded using the Docker-based SWE-bench evaluation procedure: it is applied to the repository state for that instance, then checked against the corresponding test suite.

\paragraph{Computation resources}
The reported compute budget is about 16k H100-80GB GPU-hours in total: 3k for reasoning, 5k for tool-use, and 8k for coding. Reasoning and tool-use runs use 16 H100-80GB GPUs, while coding runs use 24 H100-80GB GPUs.

\paragraph{Baselines}
\label{append-sec:Compared_methods}
Unless otherwise stated, the method comparisons below use the hardest asynchronous setting, $\mathrm{lag}=8$, because that is where freshness mismatch is most visible. We use the following names throughout the section. \textbf{Sync. OPD} is the fully synchronized on-policy baseline. \textbf{Async. OPD} is the bare overlapped pipeline with no freshness-aware control. \textbf{Async. OPD + fixed refresh} keeps the bare asynchronous objective but deterministically refreshes the active buffer every 10 optimizer steps. \textbf{Async. OPD + lag-only freshness weighting} replaces the full freshness score with
\[
f_i^{\mathrm{lag\text{-}only}} = (\tau_i+1)^{-1},
\qquad
w_i = \sigma(f_i^{\mathrm{lag\text{-}only}}-\xi),
\]
so it tests whether update lag alone is an adequate surrogate for freshness. \textcolor{black}{To evaluate the effect of directly extending trust-region methods from RL to OPD, we implement \textbf{Async. OPD + KL} (KL-based trust-region rollout anchoring) and \textbf{Async. OPD + clipped IW} (clipped importance weighting) to correct for student policy drift.}

\section{Additional Results
\label{app:add_res}}

\subsection{Entropy dynamics across lags}
Fig.~\ref{fig:appendix-entropy-lag-compare} supplements Fig.~\ref{fig:Fig1} with task-wise entropy dynamics across different lags. We can see that reasoning remains comparatively stable even as lag increases, while tool-use and coding agent progressively develop wider late-training uncertainty bands. The upper side of those bands expands especially quickly at larger lag, indicating that some runs retain high-entropy exploration, whereas others partially collapse.

\begin{figure*}[t]
\centering
\includegraphics[width=\textwidth]{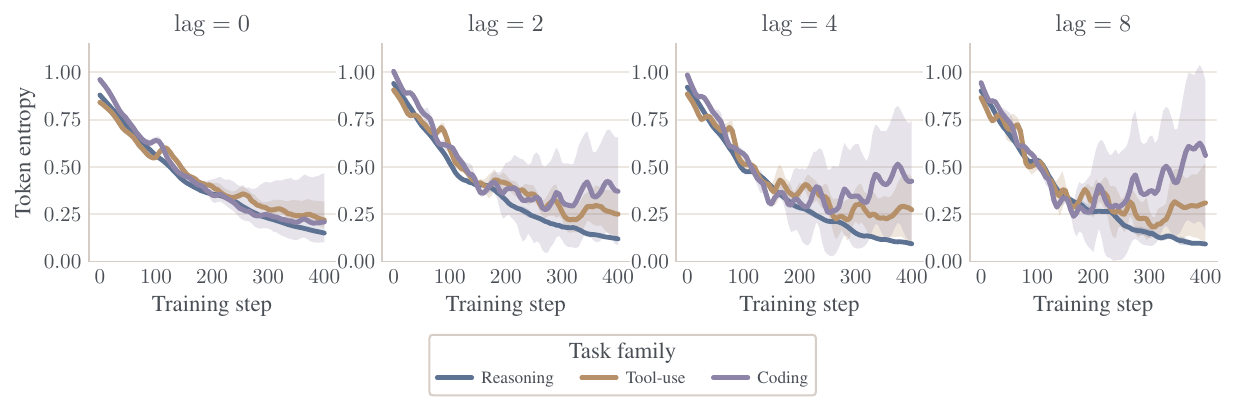}
\caption[Supplementary entropy dynamics under fixed lag values.]{Supplementary entropy dynamics for three evaluated tasks across 4 different lags. Lines show five-seed means, and shaded regions show standard deviation.}
\label{fig:appendix-entropy-lag-compare}
\end{figure*}

\subsection{Generalization Beyond Qwen Models and Vanilla OPD}
\paragraph{Non-Qwen model families}
To assess generalization beyond Qwen, we evaluate two non-Qwen student--teacher pairs under complementary settings: (i) capacity-gap distillation with Llama-3.1-8B-Instruct/Llama-3.1-70B-Instruct~\cite{grattafiori2024llama}, and (ii) same-scale capability transfer with DeepSeekMath-7B-Base/DeepSeekMath-7B-RL~\cite{shao2024deepseekmath}. We evaluate both pairs on GSM8K~\cite{mirzadeh2025gsm}, MATH500~\cite{hendrycks2021measuring}, and AMC23 using Avg@1 accuracy. As shown in Table~\ref{tab:cross-family}, asynchronous OPD substantially degrades performance in both settings, while $f$-OPD consistently recovers most of the loss and achieves performance comparable to synchronous OPD across all three benchmarks.

\begin{table}[t]
\centering
\small
\setlength{\tabcolsep}{3pt}
\caption{Avg@1 accuracy (\%) on non-Qwen student--teacher pairs.}
\label{tab:cross-family}
\begin{tabular}{llcccc}
\toprule
Student & Teacher & Method & GSM8K & MATH500 & AMC23 \\
\midrule
\multirow{3}{*}{Llama-3.1-8B-Instruct} & \multirow{3}{*}{Llama-3.1-70B-Instruct}
& Sync. OPD   & 79.2 & 23.0 & 15.0 \\
& & Async. OPD  & 63.8 & 16.1 & 7.5 \\
& & $f$-OPD & 78.1 & 22.7 & 15.6 \\
\midrule
\multirow{3}{*}{DeepSeekMath-7B-Base} & \multirow{3}{*}{DeepSeekMath-7B-RL}
& Sync. OPD   & 78.5 & 39.6 & 20.0 \\
& & Async. OPD  & 61.2 & 27.4 & 11.8 \\
& & $f$-OPD & 79.0 & 40.2 & 19.4 \\
\bottomrule
\end{tabular}
\end{table}

\paragraph{Compatibility with OPD variants}
To assess compatibility with different OPD objectives, we extend two representative variants to asynchronous execution, with and without $f$-OPD: Entropy-aware OPD (EOPD)~\cite{jin2026entropy} for mathematical reasoning and Step-wise OPD (SOD)~\cite{zhong2026sod} for tool-integrated reasoning. We follow the original settings and evaluation metrics of each method. As shown in Table~\ref{tab:opd-variants}, $f$-OPD recovers most of the performance lost under asynchronous execution for both variants, while introducing only 5\% and 8\% throughput overhead for EOPD and SOD, respectively. For example, on AIME24, it improves EOPD from 21.2 to 32.1, approaching the reported synchronous result of 35.8. These results demonstrate that freshness-aware control is complementary to different OPD objectives.

\begin{table}[t]
\centering
\small
\setlength{\tabcolsep}{3pt}
\caption{Compatibility of $f$-OPD with OPD variants. Synchronous results are from the original papers; throughput is normalized to the asynchronous baseline.}
\label{tab:opd-variants}
\begin{tabular}{lcc@{\quad}lcc}
\toprule
\multicolumn{3}{c}{EOPD (Qwen3 8B $\to$ 8B-Base)} & \multicolumn{3}{c}{SOD (Qwen3 4B $\to$ 0.6B)} \\
\cmidrule(r){1-3}\cmidrule(l){4-6}
Method & AIME24 (no-tool) & Throughput & Method & AIME24 (tool) & Throughput \\
\midrule
Sync. OPD$^\dagger$ & 35.8 & -- & Sync. OPD$^\dagger$ & 20.84 & -- \\
Async. OPD & 21.2 & 1.00$\times$ & Async. OPD & 17.6 & 1.00$\times$ \\
$f$-OPD & 32.1 & 0.95$\times$ & $f$-OPD & 19.3 & 0.92$\times$ \\
\bottomrule
\end{tabular}

\vspace{1mm}
\raggedright
\footnotesize
$^\dagger$ Synchronous results are reported from the original papers and are not used for throughput comparison.
\end{table}

\subsection{Hyperparameter Sensitivity}
\label{app:ablation-sensitivity}

We study the sensitivity of $f$-OPD to its key hyperparameters on mathematical reasoning (AIME25, avg@4) with two grid searches: (i) the drift coefficients $\alpha$ and $\beta$ in the freshness score (Eq.~\ref{eq:freshness}), which balance rollout and supervision drift; and (ii) the refresh thresholds $\kappa_f$ and $\kappa_d$, where $\kappa_d$ denotes the shared threshold of $\kappa_{\mathrm{roll}}$ and $\kappa_{\mathrm{sup}}$ used throughout the paper. The freshness threshold $\xi$ is fixed across all three tasks. When sweeping one group, the other is kept at its default value ($\alpha=1.0$, $\beta=0.5$, $\kappa_f=0.10$, $\kappa_d=0.5$).

Table~\ref{tab:sensitivity} shows that rollout and supervision drift provide complementary signals: removing either channel ($\alpha=0$ or $\beta=0$) consistently degrades performance, and disabling both drift terms gives the lowest accuracy (30.8). The best configuration ($\alpha=1.0$, $\beta=0.5$) places relatively more weight on rollout drift, which is consistent with rollout drift playing a larger role in short-horizon mathematical reasoning. For the refresh thresholds, accuracy varies only between 33.9 and 35.9 over the entire grid, well above vanilla asynchronous OPD (25.8). This indicates that $f$-OPD is robust to moderate hyperparameter variations and is not tuned to a fragile operating point.

\begin{table}[h]
\centering
\caption{Hyperparameter sensitivity of $f$-OPD on AIME25 (avg@4 accuracy, \%). Left: drift coefficients $(\alpha, \beta)$. Right: refresh thresholds $(\kappa_f, \kappa_d)$.}
\label{tab:sensitivity}
\begin{tabular}{lccc}
\toprule
 & $\beta=0$ & $\beta=0.5$ & $\beta=1.0$ \\
\midrule
$\alpha=0$   & 30.8 & 31.7 & 32.6 \\
$\alpha=0.5$ & 33.2 & 34.8 & 34.2 \\
$\alpha=1.0$ & 34.1 & \textbf{35.8} & 35.4 \\
\bottomrule
\end{tabular}
\hspace{2em}
\begin{tabular}{lccc}
\toprule
 & $\kappa_f=0.05$ & $\kappa_f=0.10$ & $\kappa_f=0.15$ \\
\midrule
$\kappa_d=0.4$ & 33.9 & 34.4 & 34.8 \\
$\kappa_d=0.5$ & 34.7 & 35.8 & \textbf{35.9} \\
$\kappa_d=0.6$ & 35.0 & 35.7 & 35.6 \\
\bottomrule
\end{tabular}
\end{table}
\color{black}

\section{Discussion and Limitations}

Our claims are deliberately scoped: asynchronous OPD becomes unstable when the student outruns its data and supervision streams, and \textit{f}-OPD reduces this mismatch through principled, freshness-aware control. Several limitations qualify this scope.

As for the limitation, first, our empirical evaluation is limited to relatively small-scale models, with the largest student and teacher models being 8B and 30B parameters, respectively. While the experiments cover diverse tasks, model families, and distillation settings, it remains to be validated whether the proposed $f$-OPD remains effective and efficient at substantially larger scales. Second, our theoretical analysis provides a discrepancy bound under bounded-regularity assumptions rather than a convergence theory. The bound characterizes how rollout and supervision staleness can contribute to the gap between asynchronous and synchronous objectives, but does not establish convergence of asynchronous training or eliminate asynchronous bias. Extending the analysis toward stronger guarantees and validating $f$-OPD at larger model scales are important directions for future work.

\section{Broader Impact}
\textit{f}-OPD offers a practical recipe for navigating the performance–efficiency trade-off in on-policy distillation, supporting more reliable post-training of large language models under realistic asynchronous pipelines. 
Beyond technical contributions, our method can lower the resource barrier for academic and resource-constrained research groups, helping democratize a direction currently concentrated in a few well-resourced institutions.


\end{document}